\documentclass[twoside,11pt]{article}
 
\usepackage{blindtext} 
\usepackage[T1]{fontenc}
\usepackage[abbrvbib,preprint]{jmlr2e}
\usepackage{amsmath}
\usepackage[dvipsnames,svgnames]{xcolor}
\usepackage{enumitem}
\usepackage{thmtools}
\usepackage{comment}
\usepackage{neuralnetwork}

\makeatletter
\newcommand{\mylabel}[2]{#2\def\@currentlabel{#2}\label{#1}}
\makeatother

\newcommand{\R}{\mathbb{R}} 
\newcommand{\N}{\mathbb{N}}
\newcommand{\Vect}{\operatorname{Vect}}
\newcommand{\supp}{\operatorname{Supp}}
\newcommand{\KL}{\operatorname{KL}}
\newcommand{\ReLU}{\operatorname{ReLU}}

\newcommand{\al}{\alpha}
\newcommand{\be}{\beta}

\newcommand{\ga}{\gamma}

\newcommand{\la}{\lambda}

\newcommand{\si}{\sigma}
\newcommand{\te}{\theta}
\newcommand{\ta}{\tau}
\newcommand{\veps}{\varepsilon}

\newcommand{\EM}{\ensuremath}

\newcommand{\cF}{\EM{\mathcal{F}}}
\newcommand{\cG}{\EM{\mathcal{G}}}

\newcommand{\cS}{\EM{\mathcal{S}}}

\newtheorem{hyp}{Assumption}

\definecolor{blendedblue}{rgb}{0.2,0.2,0.7}

\newcommand{\sbl}[1]{{\color{blendedblue}{#1}}}

\newcommand{\br}{\mathbf{r}}

\usepackage{lastpage}
\jmlrheading{}{}{}{}{}{}{ Isma{\"e}l Castillo and
Paul Egels}

\ShortHeadings{DNNs with heavy-tailed weights}{Castillo and Egels}
\firstpageno{1}

\begin{document}

\title{Posterior and variational inference for deep neural networks with heavy-tailed weights}

\author{\name Isma\"el Castillo \email ismael.castillo@sorbonne-universite.fr \\
       \addr LPSM, Sorbonne Universit\'e\\
       4, place Jussieu\\
75005, Paris, France
 \AND
  \name Paul Egels \email paul.egels@sorbonne-universite.fr \\
       \addr LPSM, Sorbonne Universit\'e\\
       4, place Jussieu\\
75005, Paris, France
 }

\editor{}

\maketitle 

\begin{abstract}%   
We consider deep neural networks in a Bayesian framework with a prior distribution sampling the network weights  at random. Following a recent idea of Agapiou and Castillo (2024), who show that heavy-tailed prior distributions achieve automatic adaptation to smoothness, we introduce a simple Bayesian deep learning prior based on heavy-tailed weights and ReLU activation. We show that the corresponding posterior distribution achieves near-optimal minimax contraction rates, {\em simultaneously} adaptive to both intrinsic dimension and smoothness of the underlying function, in a variety of contexts including nonparametric regression, geometric data and Besov spaces. While most works so far need a form of model selection built-in within the prior distribution, 
a key aspect of our approach is that it does not require to sample hyperparameters to learn the architecture of the network. We also provide variational Bayes counterparts of the results, that show that mean-field variational approximations still benefit from near-optimal theoretical support.  
\end{abstract}

\begin{keywords}
Bayesian deep neural networks, fractional posterior distributions, heavy-tailed priors, Variational Bayes, overfitting regime.
 \end{keywords}

\section{Introduction}

The last decade has seen a remarkable expansion of the use of deep neural networks (DNNs) through a broad range of applications such as imaging, natural language processing, inverse problems to name a few. In parallel to this spectacular empirical success, theory to provide understanding of the performance of these methods is emerging. Among the key elements playing a role in the mathematical and statistical analysis of DNNs are their approximation properties (e.g. \cite{yarotsky17, JSH, kohler_full}), the choice of parameters in particular the network's {\em architecture}, and the convergence of the sampling or optimisation schemes (e.g. gradient descent, variational criteria etc.) involved in the training. 

Using a Bayesian approach for inference with deep neural networks is particularly appealing, among others for the natural way to quantify uncertainty through credible sets of the posterior distribution. A first contribution on theoretical understanding in this direction is the work by \cite{polson2018posterior}, who show, building up in particular on the approximation theory derived in \cite{JSH}, that assigning spike--and--slab prior distributions to the weights of a DNN with ReLU activation leads to a near-minimax posterior convergence rate in nonparametric regression. While from the mathematical perspective spike--and--slab priors can be seen as a form of theoretical ideal, they can suffer from a particularly high sampling complexity; indeed, sampling from the associated posterior distributions often faces combinatorial difficulties, as one needs to explore a high number of possible candidate models. 
While a possible answer, as we discuss in more details below, is to deploy a Variational Bayes (VB) approximation, it is particularly desirable (and even if VB is chosen as an approximation) to develop simpler prior distributions, that still retain most mathematical properties but are easier to implement in practice. 

A number of recent contributions have considered prior distributions beyond spike--and--slab (SAS) priors 
 for the DNN's weights. From the theoretical perspective let us cite the work by \cite{lee2022asymptotic}, who broadens the results for % spike--and--slab 
  SAS of   \cite{polson2018posterior} and considers a class of shrinkage priors (whose characteristics depend on smoothness parameters of the function to be estimated); the work by  \cite{kong2023masked} considers masked Bayesian neural networks, where the `mask' determines the position of zero weights, and which allows for more efficient reparametrisation in view of computations compared to SAS posteriors. The preprint \cite{kongkim24} uses non-sparse Gaussian priors on weights, building on the  approximation theory of \cite{kohler_full}. 
  From the algorithmical perspective, the work by \citet{ghosh2019model} considers using horseshoe priors on weights. More discussion on the previously mentioned priors and their  relation to the present work can be found  in Section \ref{sec:disc} below.
  
Let us mention also two lines of work in Bayesian deep learning that are somewhat different from our approach but with natural connexions, in particular in terms of the tails of the distributions arising in the networks.  The first considers DNNs in the large width limits, following seminal ideas of \cite{neal96}: using Gaussian priors on weights, \cite{sellsingh23} use the approximation by a Gaussian process obtained obtained in large width to propose posterior sampling schemes; the work by \cite{caronetal23} instead considers the large width limit with heavier-tailed distributions on weights, obtaining a mixture of Gaussian processes in the limit.
The second line of work has no activation functions involved, but models directly  random functions accross layers, which leads to deep Gaussian processes as introduced in \cite{damianou13} and recently investigated in terms of convergence by \cite{fsh23} and \cite{cr24}.

Variational Bayes  is a particularly popular approach in machine learning and statistics, where the idea is to approximate a possibly complex posterior distribution by an element of a simpler class of distributions, effectively requiring to solve an optimisation problem, see e.g. \cite{Blei_2017}. This is a method of choice for approximating posteriors in deep neural networks, with a statistical-computational trade-off through the choice of the variational class, which is often taken to be mean-field, thereby requiring a form of independence along a certain parametrisation. Following general ideas from \cite{alquier2020} (see also \cite{yang2020alphavariational, zhang2020convergence}), the work by \cite{cherief-abdellatif20a} proves that taking a mean-field variational class combined to  the spike-and-slab prior of \citet{polson2018posterior} leads to a variational posterior that converges at the same near-optimal rates towards the true regression function. Simulations using stochastic gradient optimisation for this variational posterior, and some extensions, are considered in \cite{baietal20}. Using simpler sieve-type priors on DNNs but requiring a form of model selection, the work \cite{ohn2024adaptive} considers adaptation within a family of variational posteriors.

For more discussion and further references on Bayesian deep learning we refer to \cite{papamarkou24position}, \cite{alquierrev24},  \cite{stf24}, Chapter 4, and to the review paper \cite{arbel2023primer}.

Let us now turn to a summary of the contributions of the paper. We prove that putting a suitably rescaled heavy-tailed distribution on deep neural network weights leads to an optimal rate of convergence (up to logarithmic factors) of the corresponding posterior distribution in nonparametric regression with random design. The best rate is automatically attained without assuming any knowledge of regularity parameters of the unknown regression function, achieving in particular {\em simultaneous  adaptation} to {\em both} smoothness and intrinsic dimensions when the regression function has a compositional structure with effective intrinsic dimensions (typically) of smaller order than the input dimension. We derive similar results for scalable (mean-field) Variational Bayes approximations of the posterior distribution; each coordinate of the VB approximation has only two parameters to be fitted, one for location and one for scale. 

To help putting the results into perspective, we note that many fundamental statistical results so far on DNNs  are often stated with an `oracle' choice of the network architecture: for wide and moderately deep DNNs for instance, if the network width is well chosen in terms of both smoothness and dimensionality parameters, \cite{JSH} and \cite{kohler_full} show that an empirical risk minimiser on neural networks of the {\em prescribed} (oracle) architecture achieves a near-optimal minimax contraction rate for compositional classes. Similarly,  \cite{suzuki2018adaptivity} shows that deep learning methods are `adaptive' to the intrinsic dimension for anisotropic smoothness classes, with adaptation meaning here that the achieved rate depends only on the intrinsic dimension (as opposed to the input dimension); however the results are achieved with a network architecture that still depends on the  smoothness and intrinsic dimension parameters, which are typically unknown in practice. Building up on the seminal approximation results of the just cited papers (as well as \cite{nakada2020adaptive} for data on geometric objects), we show that well-chosen heavy tailed weight distributions enable one to obtain statistical adaptation, simultaneously in terms of smoothness and dimension parameters, using Bayesian fractional posterior distributions or their mean-field variational counterparts. Beyond classes of compositions in  nonparametric random design regression, to illustrate the flexibility of our method we consider two other applications: geometrical data through the use of the Minkowski dimension as in \cite{nakada2020adaptive} and anisotropic classes as in \cite{suzuki2021deep}. For both we show that simultaneous statistical adaptation is achieved for fractional posteriors and mean-field VB.\\

{\em Outline of the paper.} In Section \ref{sec:bnn}, we introduce the statistical framework, deep neural networks, and the proposed method; the heavy-tailed prior distribution on network weights is defined, as well as the fractional posterior distribution and its variational approximation. In Section \ref{sec:main} we present our main results, declined along three settings: compositional classes in Sections  \ref{sec:rate}--\ref{sec:ratevb}, geometric data in Section \ref{sec:geom} and anisotropic classes in Section \ref{sec:anis}; in each setting results are derived for the fractional posterior distribution, and a mean-field VB approximation. in Section \ref{sec : architecture}, we discuss the impact of the network architecture, particularly its depth. Then, we also explore extensions of our results in multiple directions: in Section \ref{sec : activ}, we consider the choice of activation function; in Section \ref{sec : tau inconnu}, we address the question of knowledge about the noise variance; and in Section \ref{sec : truepost}, we examine the case of  the standard (non-fractional) posterior. Section \ref{sec:disc} contains a discussion on possible choices of prior for Bayesian deep neural networks, discussion on sampling and algorithms and future work and open questions. Section \ref{sec:proofs} contains the proof of the main results. The Appendix includes a number of short technical lemmas, a lemma on bounds of weights of neural networks and proofs for all possible extensions of our main results. \\

\emph{Notation.}  For any real number $\beta$, let $\lfloor \beta \rfloor$ be the largest integer strictly smaller than $\beta$. We write the set of positive real numbers as $\R_{>0}$. For any two sequences we write $a_n \lesssim b_n$ for inequality up to a constant and $a_n \asymp b_n$ whenever $a_n \gtrsim b_n$ is also satisfied. For two real numbers $a,b$, we write $a \vee b=\max(a,b)$ and $a \wedge b=\min(a,b)$. For $T \geq 1$ an integer, $[T]$ denotes the set of integers $\{ 1 , \dots , T \}$. For $1 \leq p \leq \infty$, we denote $|\,\cdot \,|_p$ (resp. $\lVert \, \cdot \, \rVert_p)$ the $\ell^p$ (resp. $L^p$) norms and $\lVert \, \cdot \, \rVert_{L^p(\mu)}$ when we want to specify the integrating measure, say $\mu$.  For $d \geq 1$ an integer, $F \geq 0$ and $D \subset \R^d $, let $\mathcal{C}_d^{\beta}(D,F)$ be the ball of $d$-variate $\beta$-Hölder functions 
\[ \mathcal{C}^{\beta}_d (D,F) := \left\{ f: D \subset \R^d \to \R: \sum_{|\mathbf{k}| < \beta} ||\partial^{\mathbf{k}}f||_{\infty} + \sum_{|\mathbf{k}| = \lfloor \beta \rfloor} \underset{\underset{x \neq y}{x,y \in D}}{\sup} \frac{|\partial^{\mathbf{k}}f(x) - \partial^{\mathbf{k}}f(y)|}{|x-y|_{\infty}^{\beta - \lfloor \beta \rfloor}} \leq F \right\}.\]
Elements of this ball have all their partial derivatives up to order $\lfloor \beta \rfloor$  bounded, and their partial derivatives of order $\lfloor \beta \rfloor$ are $(\beta - \lfloor \beta \rfloor)$-Hölder-continuous. We denote by $\KL(P,Q)$ and $D_\al(P,Q)$ respectively the Kullback-Leibler and Rényi divergence of order $\al$ between probability measures $P$ and $Q$ (see Appendix \ref{app: post} for their definitions).

\section{Bayesian deep neural networks}\label{sec:bnn}

Let us now introduce the notation and statistical framework used to express our results. 

\subsection{Statistical framework}\label{sec:model}
We consider the nonparametric random design Gaussian regression model with possibly large (but fixed) dimension $d\ge 1$: one observes $n$ independent and identically distributed (iid) pairs of random variables $(X_i,Y_i) \in [0,1]^d \times \R$, $1\le i\le n$, with, for $\ta_0>0$,
\begin{equation}\label{model}
Y_i = f_0(X_i) + \tau_0\xi_i.
\end{equation}
The design points $(X_i)$ are iid draws  from a distribution $P_X$ on $[0,1]^d$ and, independently, the noise variables $(\xi_i)_i$ are iid with  $\mathcal{N}(0,1)$ distribution. From now on and unless stated otherwise we assume for simplicity that $\tau_0$ is known and take $\tau_0 =1$. In Section \ref{sec : tau inconnu} we  extend our results to the case where the noise level $\tau_0$ is unknown.  
The function $f_0$ belongs to a parameter set $\mathcal{F}$ of functions to be specified below.  Writing $P_{f_0}$ for the distribution of $(X_1,Y_1)$, the joint distribution of the observations $(X,Y):= \{(X_i,Y_i)_i\}$ is $P_{f_0}^{\otimes n}$. The inferential goal here is to recover the `true' function $f_0$ from the observed data $(X_i,Y_i)_i$, and our contraction rate results will be given in terms of the  integrated quadratic loss
\[ ||f_0 - f ||_{L^2(P_X)}^2 = \int (f_0 -f)^2 \, dP_X.\]
The minimax rate of estimation of a $\beta$--smooth function $f_0$, in the H\"older sense say,  in squared-integrated loss in this model is known to be of order $n^{- 2\beta / (2 \beta +d)}$, which suffers from the curse of dimensionality, as this rate becomes very slow for large $d$ unless $f_0$ is very smooth. In recent years  regularity classes based on compositions of functions have emerged as a benchmark for the performance of deep learning methods: indeed, these classes enable to model smaller effective dimensions compared to the input dimension $d$, with optimal rates, to be described in more details below, that only depend on the intrinsic dimensions at the successive steps of the composition.\\

\emph{Compositional structure}. Let $q \geq 1$ be an integer, $K>0$, $\mathbf{d} = (d_0, \dots , d_{q+1})$, $\mathbf{t} = (t_0, \dots t_q)$ vectors of ambient and effective dimensions such that $t_i \leq d_i$ and $\boldsymbol{\beta} = (\beta_0, \dots , \beta_q)$ a vector of smoothness parameters. Following \cite{JSH}, define a  class of compositions as
\begin{align*}
 \mathcal{G}(q, \mathbf{d}, \mathbf{t}, \boldsymbol{\beta}, K) :=  \Big\{ f & =  g_q \circ \dots \circ g_0 \, : \, g_i =(g_{ij})_{j \in [d_{i+1}]} : [a_i,b_i]^{d_i} \to [a_{i+1},b_{i+1}]^{d_{i+1}},  \\
 &  g_{ij} \in \mathcal{C}_{t_i}^{\beta_i}([a_i,b_i]^{t_i},K) , \quad |a_i|, |b_i| \leq K \Big\},
 \end{align*}
where we denote with a slight abuse of notation $\mathcal{C}_{t_i}^{\beta_i}$ the Hölder ball of functions $g_{ij} : [a_i,b_i]^{d_i} \to \R$ that can be written
\[g_{ij}(z_1, \dots,z_{d_i}) = h_{ij}(z_{q_1},\dots,z_{q_{t_i}}),\]
for $\{q_1,\dots,q_{t_i}\}$ a subset of $t_i$ indices among $\{1, \dots , d_i\}$ and $h_{ij} : [a_i,b_i]^{t_i} \to \R$ a Hölder function. We also set $d_0 =d$ and $d_{q+1} = 1$ (for an alternative definition, see also \cite{fsh23}). Let us emphasize that even if the $g_{ij}$'s take as input a vector of size $d_i$, the value of the output depends on at most $t_i\le d_i$ variables, which allows one to model low-dimensional structures. We further  note that a function $f_0$ from $\mathcal{G}(q, \mathbf{d}, \mathbf{t}, \boldsymbol{\beta}, K)$ does not have a unique decomposition as a composition of $g_i$'s. This is not a concern here, as we are solely interested in estimating $f_0$ and not $g_i$'s themselves. Many models, such as additive models and sparse tensor decompostions, fall within this framework, see e.g.  \cite{JSH}, Section 4, for more discussion.

The minimax rate of estimation of $f_0$ over the above compositional class, in terms of the quadratic loss as defined above turns out to depend only on the effective dimensions $t_i$ and `effective smoothness' parameters, for $0\le i\le q$,
\begin{equation}\label{def : regeff}
    \beta_i^* := \beta_i \prod_{k=i+1}^q (\beta_k \wedge 1).
\end{equation}
An intuition behind this  is that H\"older regularities larger than $1$ cannot in general `propagate' through inner layers of composition to make the overall H\"older smoothness larger than individual regularities of the components, but regularities smaller than one can decrease the overall smoothness.

\cite{JSH}, Theorem 3, shows that if $t_i \leq \min (d_0, \dots ,d_{q+1})$ for all $i$, then 
\begin{equation}\label{def : rateeff}
    \phi_n^* : = \underset{0 \leq i \leq q}{\max} \left\{ n^{- \frac{\beta_i^*}{2 \beta_i^* + t_i}} \right\}.
\end{equation}
is a lower bound for the minimax rate of estimation in the regression model for compositional functions and squared-integrated loss.

\subsection{Deep ReLU neural networks}\label{sec:dnn} 
The rectified linear unit (ReLU) activation function is defined for any real $x$ as 
\[\rho(x) := \max(0,x).\]
When $x$ is a multi-dimensional vector, we simply apply $\rho$ coordinate--wise. Throughout most of this paper we focus on DNNs  with ReLU activation as it is one of the most commonly used and studied activation functions in deep learning. We provide a simple way to extend our results to  other smooth enough activation functions in Section \ref{sec : activ}. Let $L\geq 1$ be an integer and $\mathbf{r}=(r_0,\dots,r_{L+1}) \in \mathbb{N}^{L+2}$, we say that a function $f : \R^{r_0} \to \R^{r_{L+1}}$ is the realisation of a ReLU neural network with architecture (or structure) $(L,\mathbf{r})$ if there is a family of matrix-vectors $\{(W_l,v_l)\}_{ l \in [L+1]}$, with $(W_l,v_l) \in \R^{r_{l} \times r_{l-1}} \times \R^{r_{l}}$ such that
\begin{equation}\label{fcomp}
f = A_{L+1} \circ \rho \circ \dots \circ \rho \circ A_1,
\end{equation}
where $A_l$ is the affine map $y \mapsto W_l y + v_l$. Here $L$ and $\mathbf{r}$ are respectively called the \emph{length} (or \emph{depth}) and the \emph{width vector} of the network  architecture. Here we take the convention that the composition in \eqref{fcomp} ends with an affine transformation with a shift (as in \cite{kohler_full}) and not simply a linear one. The network weights are the entries of the matrices $(W_l)$ and shifts vectors $(v_l)$ of the successive layers. The total number of  weight parameters of a network with architecture $(L,\mathbf{r})$ is 
\begin{equation}\label{T}
    T =T(L,\mathbf{r}) := \sum_{l=0}^L r_l r_{l+1} + \sum_{l=1}^{L+1} r_l.
\end{equation}
Let $\theta := (\theta_k)_{1 \leq k \leq T} \in \R^T$ be a vector containing all the network weights written in a given order (the choice of the order does not matter, but is fixed once and forall). Given a family of matrices-vectors $\{(W_l,v_l)\}_{ l \in [L+1]}$ specifying a neural network, to simplify the notation it is sometimes easier to consider the shift vector $v_l$ to be the $0$-th column of $W_l$ so that the set of parameters is $( \theta_k \, : \, k \in [T] ) \equiv ( W_l^{(ij)} \, : \, l \in [L+1]\,,\, (i,j) \in \{1,\dots,r_{l} \} \times \{0, \dots , r_{l-1} \} )$. This is made rigorous from the fact that
\[ \begin{pmatrix}
     v^{(1)}& W^{(11)}& \dots & W^{(1r_{l-1})}\\
     \vdots  & \vdots & & \vdots\\
     v^{(r_l)} &  W^{(r_l 1)} & \dots & W^{(r_l r_{l-1})}
\end{pmatrix} \begin{pmatrix}
    1 \\
    x^{(1)}\\
    \vdots \\
    x^{(r_{l-1})}\\
\end{pmatrix} = Wx+v.\] 
{\em The set of networks $\mathcal{F}(L, \mathbf{r})$ of architecture $(L, \mathbf{r})$.} 
Let $\mathcal{F}(L, \mathbf{r})$ be the set of all realisations of ReLU neural networks with structure $(L, \mathbf{r})$, for $L\ge 1$ a given depth and $\mathbf{r}$ a vector of widths.  Also define $\mathcal{F}(L, \mathbf{r}, s)$ as the set of realisations of neural networks that have at most $s$ active parameters, that is $\mathcal{F}(L, \mathbf{r}, s) := \left\{ f \in \mathcal{F}(L, \mathbf{r}) \, , \, \#\{ k \in [T] \, ; \, \theta_k \neq 0 \} \leq s \right\}$, with $\#S$ the cardinality of a finite set $S$.

\subsection{Heavy-tailed priors on deep neural networks} \label{sec:htprior}

\emph{Frequentist analysis of tempered posteriors. } The reconstruction of $f_0$ from the data $(X_i,Y_i)_i$ given by the regression model \eqref{model} will be conducted through a (generalised--) Bayesian analysis. From a prior distribution $\Pi$ on a (measurable) parameter space $\mathcal{F}$ containing $f_0$, given $\alpha \in (0,1)$, one can construct a data-dependent probability measure on $\mathcal{F}$ called tempered posterior distribution, given by, for any measurable $A \subset \mathcal{F}$,
\begin{equation} \label{fracpost}
 \Pi_{\alpha}[A \, | \, X,Y] := \frac{\int_{A} e^{  - \frac{\alpha}{2} \sum_{i=1}^n (Y_i - f(X_i))^2} d \Pi(f)}{\int e^{  - \frac{\alpha}{2} \sum_{i=1}^n (Y_i - f(X_i))^2} d \Pi(f)}. 
\end{equation} 
Setting $\alpha$ equal to $1$ in the last display leads to the classical posterior distribution, the conditional distribution of $f$ given $X,Y$ in model \eqref{model}, when $f$ is equipped with a prior distribution $\Pi$. Here we (mostly) work with a tempered (or `fractional') posterior instead, which corresponds to rising the Gaussian likelihood to a power $\alpha$, here taken to be smaller than $1$, which effectively downweights the influence of the data; a result for the classical posterior ($\al=1$) and an augmented prior is discussed in Section \ref{sec : truepost}. Tempered posteriors are particular cases of so-called Gibbs posteriors, where the (log--)likelihood is itself replaced by a user-chosen empirical quantity; both objects are particularly popular in both machine learning and statistical applications, see e.g. \cite{l2023semiparametric} (or the more general references \cite{ghosal_van_der_vaart_2017, alquierrev24}) for more discussion. The reason we consider fractional posteriors with $\alpha<1$ is mainly of technical nature and related to our choice of prior below; we refer to the discussion in Section \ref{sec:disc} for more on this.

To analyse the convergence of the fractional posterior distribution, we take a frequentist approach, which means that we analyse the behaviour of $\Pi_\alpha[\cdot \, | \, X,Y]$ in \eqref{fracpost} under the assumption that the data has effectively been generated according to model \eqref{model} with $f_0$ a fixed `true' regression function. The goal is then  to characterise the rate at which the fractional posterior distribution concentrates its mass around the true function $f_0$ (if at all). It turns out that a sufficient condition for this is, for contraction in terms of a R\'enyi--type divergence, that the prior distribution puts enough mass on a certain neighborhood of $f_0$ (see e.g. \citet{bhattacharya2019}, \citet{l2023semiparametric} and references therein); the version we use in the present paper can be found as Lemma \ref{lem : conc} in Appendix \ref{app: post} below. 
 For more context and references on general theory of frequentist analysis of (possibly generalised--) posterior distributions, we refer to the book by \citet{ghosal_van_der_vaart_2017} and the monograph \cite{stf24}.\\

{\em Known upper-bound on $f_0$.} As is common in the analysis of deep learning algorithms, we assume for simplicity that an upper-bound $M_0$ on the true regression function  $f_0$ is known: $\|f_0\|_\infty\le M_0$. Our results continue to hold in slightly weaker form if this is not the case.

\medskip
\emph{Heavy-tailed priors on network coefficients}. For a given depth $L$ and width vector $\br$, we define a prior distribution on $\mathcal{F}(L,\mathbf{r})$ as follows. Importantly, in our method, the choice of $L, \mathbf{r}$ are {\em deterministic} and given in \eqref{logdepth}, \eqref{compwi} (or \eqref{largewi}) below.

Let $T$ be the total number of parameters in the network, and recall the notation $[T]=\{1,\ldots,T\}$ and that we order network parameters in a fixed arbitrary order.  The heavy-tailed prior samples the network coefficients as 
\begin{equation}\label{def : prior}
    \forall k \in [T], \qquad \theta_k  = \sigma_k \, \zeta_k,
\end{equation} 
\begin{itemize}
\item where the scaling factors $\sigma_k$ for $k \in [T]$ are positive real numbers to be chosen below;
\item the random variables $\zeta_k$'s for $k \in [T]$ are independent identically distributed   with a given heavy-tailed density $h$ on $\R$ that satisfies the following conditions (H1) to (H3): there exist constants $c_1, c_2>0$ and $\kappa \geq 0$ such that
\begin{enumerate}%[label=\upshape(\Roman*)]
    \item[\mylabel{H1}{(H1)}]  \text{$h$ is symmetric, positive, bounded, and decreasing on $[0,+\infty)$,} 
    
    \item[\mylabel{H2}{(H2)}] for all $x \ge 0$,
    \begin{equation*} 
        \log \left( 1/h(x) \right) \leq c_1 (1+\log^{1+\kappa}(1+x)),
    \end{equation*}
    \item[\mylabel{H3}{(H3)}]  for all $x \ge 1$,
    \begin{equation*} 
         \overline{H}(x) := \int_x^{+\infty} h(u)du \leq \frac{c_2}{x}.
    \end{equation*} 
\end{enumerate}
\end{itemize}
Let us now give typical examples for $(\sigma_k)$ and $h$ as well as some intuition behind this choice of prior. The conditions on $h$ accommodate many standard choices of densities with polynomial tails, such as Cauchy and most Student densities: all these verify (H2) with $\kappa=0$ for a sufficiently large constant $c_1>0$; a typical possible choice of $\sigma_k$ allowed in all theorems below is $\exp(-\log^{2(1+\delta)}{n})$ for $\delta>0$, independent of $k$ (i.e. one takes this same value for all network coefficients), with a decrease in $n$ just slightly faster than a polynomial  (which would correspond to $\exp(-a\log{n})$). %Some comments are in order. 

This type of priors has been recently introduced by \cite{ac23} for  use in a number of classical nonparametric models such as regression, density estimation and classification, where it is shown to achieve optimal convergence rates (up to log factors) for tempered posteriors with automatic adaptation to smoothness; efficient simulations using infinite-dimensional MCMC schemes are also considered therein.

There are two key ideas behind this choice of prior. First, the heavy tailed distribution enables one to model large values if needed (here `large' will mean at most of the order of a polynomial in $n$); second, the scaling factors are taken sufficiently small to prevent overfitting. Unlike for spike-and-slab priors (e.g. \citet{polson2018posterior}, \cite{kong2023masked}) the prior sets no coefficients exactly to zero. This is often an advantage computationally, as modelling exact zero coefficients in the prior typically requires performing a form of discrete model selection in the posterior sampling, which can have a large computational cost even for moderate sample sizes. Note also that the architecture of the network is non-random: we refer to Section \ref{sec:main} for two possible specific choices of $L, \br$. 
 This is in contrast to model-selection type priors where a prior is put on $\br, L$ or both. Again, performing model selection for instance on the width $\br$ is a form of hyper-parameter selection problem which often leads to additional computational cost, as one needs to obtain the posterior on the hyperparameter; see also the discussion section in \cite{ac23}, where similar comments are made with respect to Gaussian process priors in regression, for which hyperparameters need to be calibrated (e.g. the key lengthscale parameter if one uses a squared-exponential kernel). \\

{\em The overall prior $\Pi$ on DNNs.} We are now in position to define the prior on functions $f$ we consider:
\begin{enumerate}
\item choose the architecture $(L,\br)$ according to one of the two deterministic choices \eqref{logdepth}--\eqref{compwi} or \eqref{logdepth}--\eqref{largewi} as specified at the beginning of Section \ref{sec:main} below; 
\item given $(L,\br)$, sample network weights $\te_k$ according the heavy-tailed prior as defined above; next form the neural network realisation $f$ as in \eqref{fcomp}.

\end{enumerate}

A final remark is that the choice of architecture parameters we make below places us in the {\em overfitting regime}: namely, our prior always fits more weights than the optimal number predicted by the classical bias-variance trade-off, see the Discussion below for more on this. \\

\subsection{Heavy-tailed mean-field tempered variational approximations}
\label{sec:vb}

Let us now turn to a variational approach, where the idea is to approximate the tempered posterior on a subset $\mathcal{S}$ of (simpler) probability distributions on the parameter space $\mathcal{F}$. The variational tempered approximation (on $\mathcal{S}$) $\hat{Q}_\al=\hat{Q}_\al(X,Y)$ is defined as the best approximation  of the tempered posterior $\Pi_{\alpha}[ \, \cdot \, | X,Y ]$ by elements of $\mathcal{S}$ in the KL--sense

\begin{equation}\label{def : tvp}
  \hat{Q}_\al :=
  \underset{Q \in \mathcal S}{\arg \! \min} \, \KL (Q , \Pi_{\alpha}[ \, \cdot \, | X,Y ]),
\end{equation} 
where the $\KL(P,Q)$ is Kullback-Leibler divergence between $P$ and $Q$ (see Appendix \ref{app: post}). Below we consider a {\em mean-field} approximation, which consists in taking a class $\cS$ of distributions of product form, assuming $\cF$ has a natural parametrisation in coordinates $(\te_1,\ldots,\te_T)$, which is the case here for DNNs by taking the collection of network weights. The individual elements of the product are taken below to belong to a scale-location family of heavy tailed distributions. 

For $h$  a (heavy-tailed) density function on $\R$ and $\mu\in\R, \varsigma>0$, we let $h_{\mu,\varsigma}$ be the shifted and re-scaled density $h_{\mu,\varsigma} : x \mapsto  h((x-\mu)/\varsigma)/\varsigma$ and set 
\[ \mathcal{H} = \mathcal{H}(h) := \left\{ Q\, : \ dQ(x) = h _{\mu , \varsigma}(x) dx,\ \, \mu \in \R, \varsigma>0 \right\}.\]

That is, $Q$ belongs to $\mathcal{H}(h)$ if it has a density with respect to the Lebesgue measure  $h_{\mu,\varsigma}$ on $\R$ for some $\mu,\varsigma$. 
 Although the class $\mathcal{H}(h)$ can be defined for an arbitrary function $h$, note  that we use the same notation `$h$'  as for the prior density from the previous section, since  in our variational results below (see Theorem \ref{thm : varcomp}), we use a shifted-and-rescaled function that matches the prior density. 
For $h$ a density and $\la>0$, we denote its $\la$--th moment by
        \begin{equation}\label{def : moment}
            m_{\lambda}(h) := \int |x|^{\lambda} h(x) \, dx.
        \end{equation}

\begin{definition}[Heavy-tailed mean-field class] \label{def : HTMF}
Let $h$ be a density function satisfying \ref{H1}, \ref{H2}, \ref{H3} with finite second order moment $m_2(h) < \infty$.
For $T\ge 1$ an integer, we define the heavy-tailed mean-field class (generated from h) as
\begin{equation}\label{eq : HTMF}
    \mathcal{S}_{HT}(h) := \left\{ Q \, : \, dQ(\theta) = \prod_{k=1}^T dQ_k(\theta_k), \text{ with } Q_k \in \mathcal{H}(h)\right\}.
\end{equation}
\end{definition}
Mean-field variational approximations such as \eqref{eq : HTMF} are particularly popular \citep{Blei_2017} in that, by removing dependencies along a decomposition of the parameter space $\mathcal{F}$ as above, they often allow for scalable computation of the minimizer defined by the optimisation problem \eqref{def : tvp}. The mean-field class \eqref{eq : HTMF} corresponds to our heavy-tailed prior on DNNs defined by \eqref{def : prior}. Variational approximation for spike-and-slab priors on neural networks have been studied previously by \citet{cherief-abdellatif20a} and \cite{baietal20}. Adaptive variational concentration results for hierarchical uniform priors have been recently obtained by \citet{ohn2024adaptive}.
Concentration tools for  variational approximations are recalled in Appendix \ref{app: post} following ideas of \citet{alquier2020}, \citet{yang2020alphavariational}, \citet{zhang2020convergence} (see also \cite{stf24} for an overview). 
For more context and discussion we also refer to the review paper on variational Bayes by \cite{Blei_2017} and the review on Bayesian neural networks  \cite{arbel2023primer}.

\section{Main results} \label{sec:main}

We now consider the convergence of tempered posteriors for heavy tailed-priors on the network weights. To illustrate the flexibility of our approach, we explore {\em simultaneous adaptation} to intrinsic dimension and smoothness  in three different settings: nonparametric regression with compositional structures, geometric data through the control of Minkowski’s dimension and anisotropic Besov spaces. For the last two, our method is the first to achieve simultaneous adaptation to dimension and smoothness to the best of our knowledge. \\

{\em A given architecture that overfits}. In learning with deep neural networks, a key aspect is the choice of architecture $(L,\mathbf{r})$. In our approach, we choose a logarithmic depth: fix $\delta > 0$ and let
\begin{equation} \label{logdepth}
 L := \lceil \log^{1+\delta} n \rceil. 
\end{equation}  
Also, we fix beforehand a (relatively) large common width for network layers by setting 
\begin{equation} \label{compwi}
\br(\sqrt{n}):=(d,\lceil \sqrt{n} \rceil, \dots , \lceil \sqrt{n} \rceil,1),
\end{equation}
with $\br(\sqrt{n})$ of size $L+2$. Below we will also consider a slightly larger choice of width 
\begin{equation} \label{largewi}
\br(n):=(d,n, \dots ,n,1),
\end{equation}
with $\br(n)$ again a line vector of size $L+2$. More generally, it is easily seen from the proofs of the results below that any choice of widths $r_l$ of polynomial order $r_l\asymp n^{a_l}$, for {\em arbitrary} given $a_l\ge 1/2$ is compatible with the obtained rates, naturally placing the prior in the {\em overfitting} regime (and already so for the smallest possible choice $a_l=1/2$). Since one main choice of prior on weights below places the same distribution on each weight, this means that the prior places no special `penalty' term for picking more coefficients.
\\  
  
{\em Choice of scaling factors $\sigma_k$.} We consider two main choices of scaling parameters $\sigma_k$ for the heavy-tailed prior on coefficients \eqref{def : prior}. The main choice discussed in the paper, and referred to as {\em constant $\sigma_k$s} (although the constant depends on $n$, but not on $k$) is, for $\delta>0$ as in \eqref{logdepth} and any $k=1,\ldots,T$,
\begin{equation} \label{fixedsig}
\sigma_k = e^{-(\log{n})^{2(1+\delta)}},
\end{equation}  
or equivalently $\log(1/\sigma_k)=\log^{2(1+\delta)}{n}$, which is independent of $k$.

We next consider a more general choice that includes \eqref{fixedsig} as a special case, and referred to as {\em directed $\sigma_k$s}. It involves a little more notation, as the network weights may now depend on their relative positions in the network (recall their ordering is fixed though arbitrary) so may be skipped at first read. Recall from the notation as in Section \ref{sec:dnn} that we interpret shift vectors $v_l$ as the $0$-th columns of the weight matrices $W_l$. Suppose we work with an architecture $(L,\br)$, 
so that the network weights $(\te_k)$ are in one--to--one correspondance with  coefficients $W_l^{(ij)}$ for $l\in[L+1]$, $i \in [r_l]$ and $j \in \{0,1,\ldots,r_{l-1}\}$.  We set, whenever $\theta_ k$ corresponds to  $W_l^{(ij)}$,
\begin{equation} \label{dirsig}
\sigma _k = \sigma_l^{(ij)},
\end{equation}
where for any admissible $l,i,j$ as above, and $\delta>0$ as in \eqref{logdepth}, 
    \begin{equation}\label{condition sigma}
        \log^{2(1+ \delta)} ( i \vee j ) \leq \log(1/ \sigma_l^{(ij)} ) \leq \log^{2(1+ \delta)} n.
    \end{equation} 
The choice \eqref{fixedsig} corresponds to taking the second inequality in \eqref{condition sigma} to be an equality. If one takes instead equality in the first inequality in \eqref{condition sigma}, one obtains a (double-index) version of the weights considered in \cite{ac23} (who set $\sigma_k=\exp(-\log^2{k})$ with one index $k$). The later choice enables one, with higher prior probability, to have larger values on a number of weights (those with small $k$), which may be useful in practice to avoid too many very small weights. 
We refer to Section \ref{sec:disc} for more discussion on the choice of weights. \\

In the sequel, for $\kappa \geq 0$ the parameter in condition (H2) on the prior density ($\kappa=0$ for polynomial tails) and $\delta>0$ as in \eqref{logdepth}, we denote 
\begin{equation} \label{defga}
\gamma = 2(1+\delta)(1+\kappa)+1. 
\end{equation}

\subsection{Adaptive contraction rate in regression} \label{sec:rate}

For a class of compositions $\mathcal{G}(q, \mathbf{d}, \mathbf{t}, \boldsymbol{\beta}, K)$ and $\ga$ as in \eqref{defga}, let us introduce the rate 
\begin{equation} \label{ratephi}
 \phi_n :=  \underset{0 \leq i \leq q}{\max}\, \left(  \frac{\log^{\gamma}n}{n}\right)^{\frac{\beta_i^*}{2 \beta_i^* + t_i}}.
 \end{equation}
This coincides up to the logarithmic factor with the optimal rate $\phi_n^*$ in the minimax sense as in \eqref{def : rateeff} for estimating a $f_0$ in the regression model if it is a composition as in the class $\cG$.

\begin{theorem}\label{thm : comp}
 Consider data from the nonparametric random design regression model \eqref{model}, with $\tau_0 =1$, $f_0 \in \mathcal{G}(q, \mathbf{d}, \mathbf{t}, \boldsymbol{\beta}, K)$ and arbitrary unknown parameters.
Let $\Pi$ be a heavy-tailed DNN prior as described in Section \ref{sec:htprior} with 
 \begin{itemize}
 \item architecture $(L,\br)$ with $L$ as in \eqref{logdepth} and $\br=\br(\sqrt{n})$ as in \eqref{compwi};
 \item network weights $\te_k =\sigma_k\zeta_k$ independent as in \eqref{def : prior}, with  
 \begin{itemize}
 \item scaling factors $\sigma_k$ verifying either \eqref{fixedsig}, or more generally \eqref{condition sigma};
 \item random $\zeta_k$ with heavy-tailed density $h$ satisfying (H1)--(H3).
 \end{itemize}
 \end{itemize}
 Then for any  $\alpha \in (0,1)$, for  $M>0$  large enough, $D_\al$  the $\al$--Rényi divergence, as $n \to \infty,$
  \[ E_{f_0} \Pi_{\alpha} \left[ \left\{ f \, : \,D_\al(f,f_0) \ge M\phi_n^2 \right\} \, |\, X,Y \right] \to 0.\]
\end{theorem}
Let us now formulate a corollary in terms of the squared integrated loss. 
For $B>0$, let $C_B(\cdot)$ be the `clipping' function $C_B(x)=(-B)\vee (x \wedge B)$. Define a clipped tempered posterior as $\Pi_\al^B[\cdot|X,Y]  = \Pi_\al[\cdot|X,Y]\circ C_B^{-1} $. Equivalently, if $f\sim \Pi_\al[\cdot|X,Y]$, then \[ g=C_B(f)\sim \Pi_\al^B[\cdot|X,Y].\] 
The next Corollary shows that the rate above in terms of R\'enyi divergence carries over to the $L^2(P_X)$ loss as long as one works with a clipped posterior, which only requires to assume that an upper bound on $\|f_0\|_\infty$ is known, which is often assumed in theory for deep learning methods (Corollary \ref{cor : clip} itself is a direct application of Lemma \ref{lem : clip} below). In the sequel, for simplicity we always state results in terms of $D_\al$, but similarly results in the $L^2(P_X)$ loss can be derived as in the next Corollary.
\begin{corollary}\label{cor : clip}
Under the setting and conditions of Theorem \ref{thm : comp}, and $\Pi_\al^B[\cdot|X,Y]$ as defined in the last display, for any $B>\|f_0\|_\infty$, any $\al\in(0,1)$, as $n\to\infty$,
 \[ E_{f_0} \Pi_{\alpha}^B \left[ \left\{ f \, : \,\| f-f_0\|_{L^2(P_X)} \ge M\phi_n \right\} \, |\, X,Y \right] \to 0.\]
\end{corollary}

Theorem \ref{thm : comp} shows that, when performing a re-scaling of every coefficient of wide structured deep neural networks, by a well-chosen factor $\sigma_k$, the heavy-tailed prior leads to nearly minimax adaptive concentration of the tempered posterior regarding both the smoothness and the hidden compositional structure of the true function $f_0$. Indeed, neither the smoothness parameters $(\beta_i)$ nor the intrinsic dimensions $(t_i)$ were used to define the prior $\Pi$. Theorem \ref{thm : mink} below shows that a similar prior also leads to adaptation to  low-dimensional structure on the input data $(X_i)$.

Theory of deep neural network approximation (see for example \citet{kohler_full} or \citet{lu2021deep}) suggests that, given a true regression function $f_0$ in a model displaying low-dimensional structure  (for example compositional structure on $f_0$) there is a `minimal' approximating network structure $(L^*(\boldsymbol{\beta},\mathbf{t}),\mathbf{r}^*(\boldsymbol{\beta},\mathbf{t}))$ in terms of depth and width. To achieve adaptation one could recover this structure via model selection with a hierarchical prior (see for example \citet{polson2018posterior}, \citet{kong2023masked},  \citet{ohn2024adaptive}). Here instead we show in Theorem \ref{thm : comp} (resp. Theorem \ref{thm : mink} below) that with our heavy-tailed prior no hierarchical step is required to achieve adaptation and it is sufficient to take a fixed (`overfitting') network architecture, independent of $\boldsymbol{\beta}$ and $\mathbf{t}$ (resp. independent of $\beta$ and $t^*$ in the result below) and draw independent coefficients. We refer to Section \ref{sec : architecture} for more discussion on the choice of the architecture.

We allow for some flexibility in the choice of the decay coefficients $\sigma_k$. One can for example chose the same decay $\sigma_k = \exp(- \log^{2(1+\delta)}n)$ for every coefficient so that the $\theta_k$'s are independent and identically distributed. We refer to Section \ref{sec:disc} for a brief discussion on other possible choices.

\begin{remark} \label{rem-thm1}
a) The rate $\phi_n$ includes a logarithmic factor which we did not try to optimise; one can make the power $\gamma$ slightly smaller by taking a slightly different depth, albeit with a slightly different choice of $\sigma_k$'s. For simplicity we do not pursue such refinements here. b) Recall that we have assumed that an upper bound $M_0$ on $f_0$ is known: this is only used through the fact that we take a `clipping' function to state Corollary \ref{cor : clip}, where $B$ has to be larger than $M_0$. If this is not assumed, then the conclusion of Theorem \ref{thm : comp} still holds in $\alpha$--Rényi divergence.  
\end{remark}

\subsection{Contraction of mean-field variational approximation}
\label{sec:ratevb}

Let us now consider taking a variational approach, and instead of sampling from the tempered posterior distribution $\Pi_\al[\cdot\, |\, X,Y]$, sample from the variational approximation $\hat{Q}_\al$ defined in \eqref{def : tvp}, where the variational class is mean-field with heavy-tailed components.
  
\begin{theorem}\label{thm : varcomp}
Consider data from the nonparametric random design regression model \eqref{model}, with $\tau_0 =1$, $f_0 \in \mathcal{G}(q, \mathbf{d}, \mathbf{t}, \boldsymbol{\beta}, K)$ and arbitrary unknown parameters. Let $\Pi$ be a heavy-tailed DNN prior as in Theorem \ref{thm : comp}, with heavy-tailed density $h$ such that  $ m_{2 \vee (1+\kappa)}(h) \leq c_3$ for some  $c_3 >0$ and $\kappa$ as in (H2). Suppose the scaling factors $\sigma_k$ verify  \eqref{fixedsig}. 
  
 Let  $\cS=\cS_{HT}(h)$ be the mean-field class generated from $h$ as in Definition \ref{def : HTMF} and let $\hat{Q}_\al$ be  the associated variational tempered posterior as in \eqref{def : tvp}.
 
Then for any  $\alpha \in (0,1)$,  for $\phi_n$ as in \eqref{ratephi}, for $M$ large enough, as $n \to \infty,$
\[ E_{f_0} \int D_{\alpha}(f,f_0) d\hat{Q}_\al(f) \le M \phi_n^2.   \] 
In particular, for any diverging sequence $M_n\to\infty$
  \[ E_{f_0}  \hat{Q}_\al\left[ \left\{ f \, : \,D_\al(f,f_0) \ge M_n \phi_n^2 \right\} \right] \to 0.\]
 \end{theorem}

 Theorem \ref{thm : varcomp} establishes contraction of the heavy-tailed mean-field variational approximation with the same adaptive rate as in Theorem \ref{thm : comp}. The definition \eqref{def : HTMF} of the variational class is independent of any unknown quantities linked to $f_0$ such as $t$ and $\boldsymbol{\beta}$, has no hierarchical structure and only requires two parameters $(m , \varsigma)$, avoiding the choice of a sparsity level hyperparameter required in model selection priors such as spike and slab priors (as in e.g. \cite{cherief-abdellatif20a, baietal20}). The moment restriction on $h$ allows for a result in expectation at the level of the variational posterior and guarantees consistency of the variational posterior mean (see also Remark \ref{rmq : markov} below), restricting however the choice of the prior to fixed decay as in \eqref{fixedsig}. Up to somewhat more technical proofs, we believe that one could also weaken the conditions to derive a result in probability  as above but without moment condition on $h$, as well as to allow for scalings verifying the more general decay \eqref{condition sigma}. For simplicity here we do not investigate such further refinements.

Similarly as for Theorem \ref{thm : comp}, one may deduce from Theorem \ref{thm : varcomp}  a convergence rate result for the variational tempered posterior $\hat{Q}_\al$ in terms of the $L^2(P_X)$ loss by post-processing $\hat{Q}_\al$ through a simple `clipping' operation.

\subsection{Geometric data} \label{sec:geom}

Rather than making assumptions about the structure of the unknown regression function $f_0$ (e.g. in terms of compositions as in Sections \ref{sec:rate} and \ref{sec:ratevb}), we can instead permit the observed data itself to exhibit a low-dimensional structure. 

For instance, one can assume that the distribution $P_X$ of the design points $(X_i)$ is supported on a smooth manifold of lower dimension than $d$. This notion is frequently employed to describe intrinsic dimensionality and various results are available to approximate functions on such manifolds using deep neural networks, see for example \cite{NEURIPS2019_fd95ec8d}, \cite{schmidthieber2019deep} and \cite{fangetal24}.
 Here we follow the more general setting of \cite{nakada2020adaptive}, which includes design points sitting on a manifold as a special case while also allowing for rougher sets (such as fractals, although for the discussion below we focus on a manifold example for simplicity) via the following notion of dimension.

\begin{definition}[Minkowski Dimension]
Let $E \subset [0,1]^d$, the $(\varepsilon , \infty)$-covering number of $E$, written $\mathcal{N}(E,\varepsilon)$ is the minimal number of $\ell_{\infty}$-balls of radius $\varepsilon >0$ necessary to cover $E$. The upper Minkowski dimension of $E$ is defined as 
\[ \dim_M E := \inf \{ t \geq 0 \, : \, \underset{\varepsilon \downarrow 0}{\lim \, \sup} \, \mathcal{N}(E,\varepsilon) \varepsilon^t = 0 \}.\]
\end{definition}

The Minkowski dimension gives insight on how the covering number of $E$ changes compared to the decrease of the radius of the covering balls. 
\begin{hyp}\label{hyp : mink}
     Suppose that $t^*:=\dim_M \, \supp P_X < d$.
\end{hyp}
Assumption \ref{hyp : mink} is satisfied in the case that $\supp P_X$ is a smooth compact  manifold of dimension $d^*$ strictly less than $d$, as then $t^*\le d^*<d$, as established in  Lemma 9 in \cite{nakada2020adaptive}.

We show below that, under Assumption \ref{hyp : mink}, our construction of heavy-tailed priors can be adapted to lead to hierarchical-free adaptation to the Minkowski dimension of the support of the design points $(X_i)$. Let us set, for $\kappa$ the constant appearing in (H2) and $t>0$,

\begin{equation} \label{vepsn}
 \varepsilon_n(t) := \left( n / \log^{2(1+ \kappa)+1}n\right)^{- \beta / (2 \beta + t)}.
\end{equation}  
 
\begin{theorem}\label{thm : mink}
 Consider data from the nonparametric random design regression model \eqref{model}, with $\tau_0 =1$, $f_0 \in \mathcal{C}_d^{\beta}([0,1]^d)$. Suppose Assumption \ref{hyp : mink} holds  
 and let $\Pi$ be a heavy-tailed DNN prior as described in Section \ref{sec:htprior} with  
 \begin{itemize}
 \item architecture $(L,\br)$ with $L = \lceil \log n \rceil$ and $\br=\br(n)$ as in \eqref{largewi};
 \item network weights $\te_k = \sigma_k\zeta_k$ independently as in \eqref{def : prior}, with  
 \begin{itemize}
 \item scaling factors $\sigma_k$ verifying \eqref{fixedsig};
 \item random $\zeta_k$ with heavy-tailed density $h$ satisfying (H1)--(H3).
 \end{itemize}
 \end{itemize}
Let  $\varepsilon_n(t)$ as in \eqref{vepsn}, for a given arbitrary $t>t^*$. Then for any  $\alpha \in (0,1)$,  for $M$ large enough, as $n \to \infty,$
 \[ E_{f_0} \Pi_{\alpha} \left[ \left\{ f \, : \, D_\al(f,f_0) \ge  M \varepsilon_n(t) \right\} \, |\, X,Y \right] \to 0.\]
 Assume further that    $m_{2 \vee (1 + \kappa)}(h) < \infty$ and  let $\hat{Q}_{\alpha}$ be the heavy-tailed tempered variational approximation defined by \eqref{def : tvp} on $\mathcal{S}_{HT}(h)$. 
 Then the contraction of $\hat{Q}_{\alpha}$ as stated in Theorem \ref{thm : varcomp}   holds here with $\phi_n$ replaced by $\veps_n$. 
\end{theorem}

As a special case, if $\supp P_X$ is a smooth compact  manifold of dimension $d^*$ strictly less than $d$, Theorem \ref{thm : mink} gives a contraction rate $\veps_n(t)$ in particular for any $t>d^*$, that is a rate $\veps_n(d^*+\eta)$ for any fixed arbitrarily small $\eta>0$, which is near-optimal (the optimal minimax rate being $\veps_n(d^*)$). The proof of Theorem \ref{thm : mink} in given in Appendix \ref{app : minkproof}.

In Theorem \ref{thm : mink}, we use the choice $\br=\br(n)$ as in \eqref{largewi}, that is, we increase the network width quadratically compared to \eqref{compwi}. Indeed, in order to establish the above  posterior concentration, we rely on an  approximation result of functions in such a setting by ReLU-deep neural networks  by \citet{nakada2020adaptive}. Such a result, recalled as Lemma \ref{lem : approxmink} below, uses sparse approximations (i.e. many network weights are equal to zero), in a similar spirit as the corresponding result in \cite{JSH}, which has the effect to increase (quadratically) the width of the network compared to Theorem \ref{thm : comp} and Theorem \ref{thm : varcomp}, for which we used the recent fully-connected approximation results of \cite{kohler_full} (as opposed to \cite{JSH}, Theorem 5).  It should also be possible to use fully-connected approximation results in the present setting, for example in the spirit of Theorem 6.3 of \citet{jiao2023deep} (which, similar to \cite{kohler_full}, has at most polynomially increasing weights).

Also, here we have considered only the case of constant scaling factors \eqref{fixedsig}. The same remarks as in the discussion below Theorem \ref{thm : varcomp} apply: the statement can presumably be adapted to directed decay of the network weights up to using a slightly more technical proof. 

A number of other geometrical settings of interest  could also   be considered within the framework of the present work. Among others, one could consider the case of design points not exactly supported {\em on} the manifold but rather in a neighborhood of it. While such extensions are very interesting, they fall beyond the scope of the present contribution. We refer to the Discussion in Section \ref{sec:disc} for some comments on how to possibly extend our results to design points lying in a tubular neighborhood of an unknown manifold.

\subsection{Anisotropic Besov classes} \label{sec:anis}

Instead of assuming the true function to be in a H\"older space as in Theorem \ref{thm : comp}, we now investigate the case that the smoothness of the true function is both depending on the location (non-homogeneous smoothness) and the direction (anisotropic). For simplicity of presentation we focus on the setting without additional compositional structure, but one could also derive results in case of compositions as in Section \ref{sec:rate}; the case of a single function is already illustrative of the adaptation properties at stake. In the following we recall a definition of anisotropic Besov spaces through the moduli of smoothness. For properties of these spaces and  equivalent characterizations (for example through multiresolution analysis) we refer to \cite{GineNickl}, Section 4.3. and to the monograph \cite{Triebel}. 
 
For a function $f : \R^d \to \R$, the $m$-th difference of $f$ in the direction $ v \in \R^d$ is 
\[ \Delta_v^m(f)(x) := \Delta_v^{m-1}(f)(x + v) - \Delta_v^{m-1}(f)(x), \qquad  \Delta_v^{0}(f)(x) := f(x).   \]

\begin{definition}\label{def : modul} For a function $f \in L^p([0,1]^d)$, where $p \in [1,\infty)$, the $m$-th modulus of smoothness of $f$ is defined by
\[ \omega_{m,p}(f,u) := \sup_{ v \in \R^d \, :\, |v_i|\leq u_i } \lVert \Delta_v^m(f) \rVert_p, \qquad \text{for } u=(u_1,\dots,u_d),\ \ u_i >0.\]
    
\end{definition}

\begin{definition}[Anisotropic Besov space] \label{def : besov} Let $p \in [1,\infty)$,  $ \beta = (\beta_1,\dots , \beta_d) \in \R_{>0}^d$ and $m := \max_i \lfloor \beta_i \rfloor +1$. Let the seminorm $|\, \cdot \,|_{B_{pp}^{\beta}}$ be defined as
\begin{equation*}
    |f|_{B_{pp}^{\beta}}^p := \sum_{k=0}^{\infty} \left[ 2^k w_{m,p}\left(f,(2^{-k/\beta_1}, \dots , 2^{-k/\beta_d}) \right)\right]^p.
\end{equation*}
For $\lVert f \rVert_{B_{pp}^\beta} = \lVert f \rVert_p + |f|_{B_{pp}^{\beta}} $ and $M>0$, the anisotropic Besov ball of radius $M$ on $[0,1]^d$ is 
\begin{equation*}
    B_{pp}^{\beta}(M) := \left\{ f \in L^p([0,1]^d) \, : \,  \lVert f \rVert_{B_{p,p}^\beta} \leq M  \right\}.
\end{equation*}    
\end{definition}
Optimal minimax rates on such anisotropic Besov balls can be expressed in terms of the harmonic mean smoothness, defined as
\begin{equation}\label{def : moyharm}
    \Tilde{\beta}^{-1}: =  \sum_{k=1}^d \beta_k^{-1} .
\end{equation} 
Let us set, for $\gamma$ the constant given by \eqref{defga} and $\Tilde{\beta}$ defined by \eqref{def : moyharm},
\begin{equation}\label{vepsn2}
    \epsilon_n = \left(  n / \log^{\gamma}n  \right)^{- \Tilde{\beta} / (2 \Tilde{\beta} + 1)}.
\end{equation}

\begin{theorem}\label{thm : besov}
 Consider data from the nonparametric random design regression model \eqref{model}, with $\tau_0 =1$, $f_0 \in B_{pp}^{\beta}(1)$ and $\beta = (\beta_1,\dots,\beta_d) \in \R_{>0}^d$. Assume $p < 2$  and $\Tilde{\beta} > 1/p$, and let $\Pi$ be a heavy-tailed DNN prior as described in Section \ref{sec:htprior} with 
 \begin{itemize}
 \item architecture $(L,\br)$ with $L$ as in \eqref{logdepth} and $\br=\br(n)$   as in \eqref{largewi};
 \item network weights $\te_k = \sigma_k\zeta_k$ independently as in \eqref{def : prior}, with  
 \begin{itemize}
 \item scaling factors $\sigma_k$ verifying \eqref{fixedsig};
 \item random $\zeta_k$ with heavy-tailed density $h$ satisfying (H1)--(H3).
 \end{itemize}
 \end{itemize}
 Then for any  $\alpha \in (0,1)$, for  $M>0$  large enough it holds, as $n \to \infty,$ for $\epsilon_n$ as in \eqref{vepsn2}
 \[ E_{f_0} \Pi_{\alpha} \left[ \left\{ f \, : \,  D_\al(f,f_0) 
 \ge  M \epsilon_n \right\} \, |\, X,Y \right] \to 0.\]
 Assume further that    $m_{2 \vee (1 + \kappa)}(h) < \infty$ and  let $\hat{Q}_{\alpha}$ be the heavy-tailed tempered variational approximation defined by \eqref{def : tvp} on $\mathcal{S}_{HT}(h)$. 
Then the contraction of $\hat{Q}_{\alpha}$ as stated in Theorem \ref{thm : varcomp}   holds here with $\phi_n$ replaced by $\epsilon_n$. 
\end{theorem}

Whenever $p < 2$ and $\Tilde{\beta} > 1/p$, functions in $B_{pp}^{\beta}$ are continuous and in-homogeneously smooth in terms of the $L^2$--norm (see for instance \citet{suzuki2021deep} and \cite{10.1214/aos/1024691081}). 
An interesting feature of Theorem \ref{thm : besov} is the adaptation to the anisotropic smoothness it provides. Anisotropy itself may help  mitigate the curse of dimensionality. Indeed, when $\beta = (\beta_0, \dots, \beta_0)$, we have $\Tilde{\beta} = \beta_0 / d$, and one obtains the usual rate for a homogeneously smooth function in $\mathbb{R}^d$. However, when $f_0$ is not very smooth only in a small number of directions, the estimation rate $n^{-\Tilde{\beta}/(2 \Tilde{\beta} + 1)}$ depends only on this specific small number of coordinates rather than on the whole ambient dimension $d$.

The proof of Theorem \ref{thm : besov} can be found in appendix \ref{app : proof besov}. Similar remarks as for Theorem \ref{thm : mink} can be made. In particular, \sbl{}. Also, the proof uses the sparse approximation of \cite{suzuki2021deep} recalled as Lemma \ref{lem : approxbesov} below, which again explains why it is natural to use the width $\br=\br(n)$  as in \eqref{largewi}.

\subsection{Influence of the network architecture}\label{sec : architecture}

In the next lines, we discuss possible choices of the architecture. We consider in particular our proposed choices (that is, \eqref{compwi} or \eqref{largewi}), but also contrast these with different ones, examining the impact on rates and optimality. For simplicity we discuss the simplest setting of estimating a $\beta$--H\"older function (without the compositional structure), first in dimension $1$ (see the end of the Section for the case of dimension $d$), that is, on $[0,1]$, for which the minimax rate in squared integrated norm is of order $\veps_n^2=n^{-2\be/(2\be+1)}$.\\

{\em Deterministic Architecture.} 
In this work, we have chosen a prior with a fixed (deterministic) width and depth. This choice aligns with common practice for deep neural networks where practitioners typically select a fixed architecture prior to learning network weights (e.g. by gradient descent or a variational optimisation algorithm). Among Bayesian approaches so far for deep neural networks, those that have claimed adaptation with respect to the smoothness parameter $\be>0$ have  involved  sampling either the nonzero coefficients, or  the architecture itself, from a hyper-prior: this is the case in \cite{polson2018posterior} where the prior is a spike-and-slab distribution that sets to zero a number of coefficients at random and very recently \cite{kongkim24}, where the width in the architecture is random. A key insight from our results is that the deterministic (and overfitting) choice of architecture, where all networks coefficients are non-zero under the prior, is not detrimental to optimality and statistical adaptation properties.\\

{\em Moderate depth.} In the present work, we have restricted the depth of the network to grow logarithmically with the sample size $n$. This restriction is primarily technical, since, as we explain in more details now, our arguments (as well as those used so far in the theoretical study of Bayesian DNNs, to the best of our knowledge) are presently difficult to extend to the `very deep' case.  
A crucial challenge in posterior contraction analysis (even in the non-tempered case, $\al =1$) is verifying the so-called `mass condition' (Lemma \ref{lem : conc}), which ensures the prior places enough probability mass near the true function. In the present setting, this is achieved using the sufficient condition in Lemma \ref{propag1}. This result quantifies how much two network realisations are far appart if we control individual distances between weights. Using this `error propagation' Lemma requires controlling the distance between the `true' coefficients (of a neural network approximation of $f_0$) and those sampled from the prior, uniformly over the network (in particular independently of the depth). Notably, the required precision for this control increases polynomially with width (for fixed depth) and exponentially with depth (for fixed width). This makes analyzing `very deep' priors significantly more challenging, as it would require having a very precise prior mass control for coefficients very deep in the network. 
We believe that such a result is beyond the scope of the present paper as it would involve arguments specifically tied to the study of very deep neural networks, but an interesting direction for future work (in particular the understanding of depth-dependence one should put on the decay coefficients when randomly initializing a neural network is of significant interest and has been studied for instance, in the case of very-deep Res-Nets in a slightly different context by \cite{marion2022scaling}).\\

{\em Choice of the width.} Having fixed the network depth $L$ to a slowly growing logarithmic term as in \eqref{logdepth}, the only degree of freedom left on the architecture is now the width of the network (the activation function is in principle another degree of freedom;  activations beyond ReLU are discussed in Section \ref{sec : activ} below).  The fully-connected approximation result from \cite{kohler_full}, recalled in Proposition \ref{KLt} below, implies that for a  logarithmic depth $L$ as above, there is an oracle width $r^* := \br(N^*)$  such that networks with architecture $(L ,r^*)$ approximate $\beta$--Hölder functions with the desired target accuracy $\veps_n$, and the corresponding width $N^*$ is proportional to $n^{1/\{2(2\beta +1)\}}$ (recall we have assumed here that $d=1$). We say that a neural network architecture is `underfitting' (resp. `overfitting') if its width parameter is smaller (resp. greater) than $r^*$.\\

{\em Overfitting vs underfitting architecture.} Consider a prior construction where the network width is deterministic and set to $n^a$ for some $a > 0$, leading to an architecture $(L, \br(n^a))$. There are two cases
\begin{enumerate}
    \item $a=1/2$ (our choice) or $a>1/2$. As shown in Theorem \ref{thm : comp}, when $a = 1/2$, the heavy-tailed prior with architecture $(L, \br(n^a))$ achieves the minimax adaptive posterior contraction rate. It is not hard to check that the result of Theorem  \ref{thm : comp} continues to go through larger values of $a$, that is $a\ge 1/2$. For $a\ge 1/2$, we are always in the {\em overfitting} regime since $n^a\gg N^*=n^{1/\{2(2\beta +1)\}}$ for any possible smoothness $\be>0$. 
    \item $a<1/2$. Then  there exists a regularity index $\be_{min}>0$ for which $a < 1/(2(2\beta_{min} +1))$ by definition and in that case the architecture is {\em undersmoothing}. In this case, it is possible to show that approximation by a ReLU-DNN with architecture $(L, \br(n^a))$ cannot be fast enough to reach the desired optimal estimation rate $\veps_n$ for $f_0$. 
Indeed, applying Lemma 1 from \cite{JSH} (a lower bound in the $L^2$-norm is also available, as noted in the remark therein), one obtains the following lower bound on the approximation rate
    \begin{equation}
        \sup_{f_0 \in \mathcal{C}_d^{\beta_{min}}([0,1])} \inf_{f \in \mathcal{F}(L, \br(n^a))} \| f - f_0 \|_{\infty} \gtrsim \eta_n,
    \end{equation}
    where $\eta_n \approx n^{-2a\beta_{min}}$, up to logarithmic factors.  Since $a < 1/(2(2\beta_{min} +1))$, this rate is strictly slower than the minimax rate $n^{-\be_{min}/(2\be_{min}+1)}$ for that regularity. The proof of Theorem \ref{thm : comp} would then give a suboptimal contraction rate.
\end{enumerate}     

{\em Our choice: an overfitting architecture yet leading to optimal rates.} In conclusion, the smallest deterministic width for the architecture (given $L$ has been chosen as above), for which one still achieves statistical adaptation to smoothness, is the overfitting choice  $n^{1/2}$ corresponding to $a = 1/2$ (our choice in Theorem \ref{thm : comp}). 
Finally, in the case of estimation on $[0,1]^d$, the oracle width within the setting of \cite{kohler_full} is $n^{d/\{2(2\be+d)\}}$, so once again the overfitting width $n^{1/2}$ is the smallest one that still verifies $n^{1/2}\ge n^{d/\{2(2\be+d)\}}$ for any value of $\be>0, d\ge 1$.

\subsection{Extension to other activation functions} \label{sec : activ}

All results so far have been stated in the case where the activation  $\rho$ in \eqref{fcomp} is the ReLU function $x \mapsto \max(0,x)$.  We now see how to extend our main first result, Theorem \ref{thm : comp}, to other, sigmoid-type, smooth, activation functions. Here we state this extension in a slightly informal way, and refer the reader to the Appendix Section \ref{app : activ} for a formal statement that adapts Theorem \ref{thm : comp} to other activation functions (Theorem \ref{thm: compactiv} therein). One can easily follow the steps developed in Section \ref{app : activ} to also adapt Theorems \ref{thm : mink} and \ref{thm : besov}.  

Within the proof of Theorem \ref{thm : comp},  only two properties are required from the activation function $\rho$:
\begin{itemize} 
    \item  a  result similar to Lemma \ref{propag1} should hold; the latter gives a control of how errors on individual weights `propagate' into an error at the level of the function realised by the neural network;
        \item there should exist a wide neural network having $\rho$ as activation function, such that the neural network approximates the function $f_0$ with the correct rate and such that the weights of the network are bounded `sub-exponentially' in that the magnitude of the weights is at most $\exp(c_1 \log^{c_2}n))$ for some positive constants $c_1,c_2$.  
\end{itemize}
Note that instead of requiring a polynomial control of the weights  in the second point above, one allows for a slightly larger growth in $n$, so that the condition is significantly milder, a fact we crucially exploit in our proof for $\rho$-activations.

Let us now comment on how to check these two properties.
Regarding the first point, it can be checked that  if the activation $\rho$ is Lipschitz-continuous, Lemma \ref{propag1} still holds up to a constant (see Lemma \ref{propag2}). To ensure the second required property on $\rho$, one could prove an approximation result for DNNs with activation $\rho$ as long as such result provides an appropriate control on the magnitude of the weights of the approximating network (this is the approach we take in Appendix \ref{app : bound} to prove Theorem \ref{thm : comp}). We propose here another approach that, under some smoothness conditions on $\rho$, enables to start from a ReLU approximation with polynomially bounded weights to obtain the required  approximation with  activation $\rho$.
 
\begin{definition}\label{def : locquad}
    A function $\rho : \R \to \R$ is said to be admissible if 
    \begin{itemize}
        \item $\rho$ is non-decreasing, bounded and Lipschitz-continuous,
        \item $\rho$ is three times continuously differentiable, with bounded derivatives and there exists a point $x_\rho \in \R$ such that $\rho'(x_\rho) \neq 0$ and $\rho''(x_\rho) \neq 0$,
        \item there exists two real numbers $a < b$, and a constant $c_\rho>0$, such that, 
        \[ \sup_{x \in \R}| x\rho(x) - \max(ax,bx)| \leq c_\rho.\]
    \end{itemize}
\end{definition}
The smoothness conditions in Definition \ref{def : locquad} ensure that $\rho$ is a sigmoid-type function. In particular the third condition states that $\rho$ should converge asymptotically at least as fast as $1/|x|$ towards the horizontal line $\{y = a\}$ (resp. $\{y=b \}$) when $x \to - \infty$ (resp. $y \to + \infty$). These conditions are satisfied by most sigmoid-type activations commonly used in deep-learning, for instance the logistic activation $ x \mapsto (1+e^{-x})^{-1}$ satisfies the conditions with $a =0$ and $b=1$. Both the hyperbolic tangent (tanh) and the error (erf) activation functions satisfy these conditions with $a =-1 $ and $b=1$.

\bigskip
\textbf{Claim.} If $\rho$ is admissible according to Definition \ref{def : locquad}, every previous Theorem still holds with $\rho$ in place of the ReLU activation.

\bigskip

Let us now give a brief intuition of the proof, referring to Appendix \ref{app : activ} and  Theorem \ref{thm: compactiv} therein for a formal statement and proof.
Noticing that for $a < b$ and any $x \in \R$, we have
\[ \ReLU(x) =\frac{1}{b-a}(\max(ax,bx) - ax),\]
If $\rho$ is admissible results from \cite{kohler2022estimation} show that the ReLU function can be approximated by a $\rho$-DNN with constant width and depth and polynomially bounded weights. From this it follows (Lemma \ref{lem : relutorho}) that any ReLU-DNN approximation results with logarithmic depth transposes to a $\rho$-DNN result with same architecture and sub-exponential weights. 
\bigskip

Let us emphasize that, in the current proof, going from a ReLU result to an admissible $\rho$ result requires one to be able to work with weights going to infinity with $n$ at a possibly slightly faster than polynomial rate (but slower than exponential). This shows that our heavy-tailed priors are also suited for a wide range of sigmoidal activation functions, which might be of interest in practice. Indeed, as we have mentioned, and as opposed to priors requiring hyper-priors on the architecture, our priors with fixed architecture are suitable for both MCMC and variational procedures; even though ReLU activations are popular in methods involving backpropagation because of the gradient being essentially a Boolean (therefore very fast to compute), in MCMC algorithms involving more complex dynamics (for example HMC \cite{nealHMC}) it might be of interest to work with a smoother activation function so that the target log-density is smooth, as suggested in \cite{dinh2024hamiltonian}.

\subsection{Extension to unknown noise level $\tau_0$}\label{sec : tau inconnu}
So far in model \eqref{model}, we have taken the noise variance $\tau_0^2$ to be equal to $1$; however in practice this quantity may be unknown. If $\tau_0$ is unknown, let us write $P_{f_0,\tau_0}$ for the distribution of $(X_1,Y_1)$ under model \eqref{model}, the parameters being now $(f, \tau^2) \in \mathcal{F} \times \R_{>0}$. We put a prior $\Pi$ on $(f,\tau^2)$ of the form $\Pi = \Pi_f \otimes \pi_{\tau^2}$ where $\Pi_f$ is a heavy-tailed DNN prior as before and $\pi_{\tau^2}$ is a distribution on $\R_{>0}$ with a density bounded away from $0$ on a neighborhood of $\tau_0$: \begin{equation}\label{eq : condpriortau}
    d\pi_{\tau^2}(x)=q(x)dx, \qquad \inf_{x\in[\ta_0^2-\epsilon,\ta_0^2+\epsilon]}q(x)\ge r_0,
\end{equation} 
for some $\epsilon>0$ and $r_0>0$, which is satisfied by standard choices of priors on variance parameters, such as exponential or inverse-gamma distributions with fixed parameters.

Contraction of the tempered posterior as stated in Theorems (\ref{thm : comp}, \ref{thm : mink} and \ref{thm : besov}) still holds true when $\tau_0$ is unknown as long as one chooses a prior on $\tau^2$ satisfying Condition \eqref{eq : condpriortau}. We only state below the generalisation of Theorem \ref{thm : comp}, the same holds true for the other Theorems. 

\begin{theorem}\label{thm : tau inconnu}
     Consider data from the nonparametric random design regression model \eqref{model}, with unknown $\tau_0 >0$ and $f_0 \in \mathcal{G}(q, \mathbf{d}, \mathbf{t}, \boldsymbol{\beta}, K)$ for arbitrary unknown parameters.
Let $\Pi = \Pi_f \otimes \pi_{\tau^2}$ be a prior on $(f,\tau^2)$ where $\Pi_f$ is the heavy-tailed DNN prior as in Theorem \ref{thm : comp} and $\pi_{\tau^2}$ is a prior on $\R_{>0}$ satisfying condition \eqref{eq : condpriortau}. Recall that $\phi_n$ is the rate as in \eqref{ratephi}. Then for any $\alpha \in (0,1)$, for $M>0$ large enough, as $n \to  \infty$,
\[ E_{f_0} \Pi_\al \left[ \left\{ (f,\tau^2) \, : \, D_\al(P_{f,\tau^2},P_{f_0,\tau_0^2}) \leq M \phi_n^2\right\} \, | \, X,Y \right] \to 1.\]
Furthermore, assuming $||f_0||_{\infty} \leq B$, the clipped posterior $\Pi_\al^B[\cdot|X,Y]$ satisfies, as $n \to \infty$,
\[ E_{f_0} \Pi_\al^B \left[ \left\{ (f,\tau^2) \, : \, ||f-f_0||_{L^2(P_X)} \leq M \phi_n \, , |\tau^2 - \tau_0^2| \leq M \phi_n\right\} \, | \, X,Y \right] \to 1.\]
\end{theorem}
The proof of Theorem \ref{thm : tau inconnu} can be found in Appendix \ref{app : tau inconnu}. We note that the second part of the statement entails posterior contraction around the unknown variance parameter $\tau_0^2$.
       
\subsection{Extension to standard posteriors ($\alpha=1$)} \label{sec : truepost}

For technical reasons we have so far focused on fractional posterior distributions with $\alpha<1$. This enables one to focus on prior mass conditions, which is particularly handy when dealing with heavy-tailed priors. Although one can conjecture (this has been shown to be true for heavy-tailed series prior in the technically easier Gaussian white noise regression model in \cite{ac23}) that all previous results continue to go through under exactly the same prior and conditions, this seems technically non-trivial and presently hard to reach with currently available techniques, see the discussion in Section \ref{sec:disc} for more on this.

However, one considers an augmented prior that also models the variance of the noise, which one would do anyways under unknown noise level as in the previous Section \ref{sec : tau inconnu}, then one can use a recent idea from \cite{cr24} (who took inspiration from an idea of \cite{mai24}, see also the references therein, in a different high-dimensional context) who show that, in the present random design regression model, augmenting a given prior on $f$ with a {\em well-chosen} prior on $\ta^2$ enables one to transfer an already obtained contraction rate for a fractional posterior on $f$ with $\al<1$ (such as ones obtained above) to a contraction rate for the standard posterior ($\al=1$) on $f$.

As in the previous Section \ref{sec : tau inconnu}, we write $P_{f_0,\tau_0}$ for the distribution of $(X_1,Y_1)$ under model \eqref{model}, the parameters being $(f, \tau^2) \in \mathcal{F} \times \R_{>0}$. We put a prior $\Pi$ on $(f,\tau^2)$ of the form $\Pi = \Pi_f \otimes \pi_{\tau^2}$ where $\Pi_f$ is a heavy-tailed DNN prior as before and $\pi_{\tau^2}=\pi_{\tau^2,n}$ the following distribution on $\R_{>0}$, for some fixed $b\in(0,1)$,
\begin{equation}\label{priorsig}
\pi_{\tau^2} = \text{Gamma}\Big( \Big\{\frac{1-b}{2}\Big\}n+1,b \big),
\end{equation}
where $\text{Gamma}(a_1,a_2)$ denotes a Gamma distribution with shape parameter $a_1>0$ and rate parameter $a_2>0$, of density proportional to $x\to x^{a_1-1}e^{-a_2 x}$ on $(0,\infty)$. This prior induces a posterior distribution $\Pi[\cdot \, | \, (X,Y)]$ jointly on $f$ and $\ta^2$. 
Similar to the result on the clipped posterior in Corollary \ref{cor : clip}, let  $\Pi^B[\cdot \, | \, (X,Y)]$ denote the clipped posterior (where the clipping function operates only on the $f$ part, not on $\ta^2$). 

The next result shows that, if the conditions of Theorem \ref{thm : comp} are satisfied for the prior $\Pi_f$, then the classical marginal posterior on $f$, that is $\Pi[f\in \cdot\,,\,\tau^2\in \mathbb{R}_{>0} \, | \, (X,Y)]$ also contracts at the same rate as the one from Theorem \ref{thm : comp}. 

\begin{theorem}  \label{thm:true}
Let $\Pi$ be a prior on $(f,\tau^2)$ of product form $\Pi_f\otimes \pi_{\tau^2}$, with $\pi_{\si^2}$ given by \eqref{priorsig}. Suppose the conditions of Theorem \ref{thm : comp} are satisfied. Then for $\Pi^B[\cdot|X,Y]$ the clipped posterior as defined above, for any $B>\|f_0\|_\infty$, 
\[ E_{f_0} \Pi^B \left[ \left\{ f \, : \,\| f-f_0\|_{L^2(P_X)} \ge M\phi_n \right\} \, |\, X,Y \right] \to 0.\]
That is, the conclusion of Corollary \ref{cor : clip} holds for the clipped marginal in $f$ of the classical posterior $\Pi[\cdot \, | \, (X,Y)]$ ($\al=1$).
\end{theorem}
Let us note that this result also automatically handles the more practical setting where, as in the last subsection, the noise variance $\ta_0^2$ is unknown. 

\section{Discussion} \label{sec:disc}

 A main key take-away from this work is that putting suitable heavy-tailed distribution on weights of deep neural networks enable automatic and simultaneous adaptation to intrinsic dimension and smoothness. This prior distribution allows for a  {\em soft selection} of relevant weights in the network, that get higher values under the posterior, while less important weights can get very small values. The approach does not use a `hard' variable selection approach, which would either set some weights to zero, and hence require the need to sample from the posterior distribution on the support points, or attempt to sample the network architecture at random, which would require to sample from the posterior on $(L,\br)$, thus requiring in both cases to sample from `support' hyperparameters, which can be costly computationally.  \\

Another key idea is that the approach allows for a wide variety of choice of {\em overfitting} architectures. We have seen, say in the setting of standard nonparametric regression to fix ideas,  that any choice of a common width for the layers that is polynomial in $n$ with a power at least $1/2$, that it $r_i\asymp n^{a}$ with $a>1/2$ is compatible with optimal rates with our method. We note that we have not investigated here the case of very deep networks with polynomial depth, but the idea could in principle apply there too. We also allow for a fairly wide variety of scalings $\sigma_k$ of the individual weights, showing that the optimality of the rates is robust to different choices of scalings.\\

\subsection{Discussion: priors for Bayesian deep neural networks}

{\em Priors for Bayesian deep neural networks.} For networks of large width (possibly increasing with the sample size, but finite) and moderate depth, a number of natural prior distributions have been proposed in recent works. We provide now a quick overview of recent results and refer to \cite{fortuin2022priors} for a recent specific review on  choices of priors for BNNs.

Let us first discuss some differences of our approach with another natural family of priors that directly induces sparsity. One approach directly applies sparsity to the weights by drawing exact zeros through spike-and-slab (SAS) priors, which were among the first to receive frequentist theoretical guarantees for full posteriors by \cite{polson2018posterior} and for their variational approximations by \cite{cherief-abdellatif20a}. A tractable variational Bayes (VB) algorithm was later proposed by \cite{baietal20} to handle these priors, allowing for a continuous relaxation of the spike-and-slab and updating the variational parameters through stochastic gradient descent. Another type of sparse prior enforces sparsity at the node level, for instance via masking functions, as explored by \cite{kong2023masked}, who also provide an MCMC algorithm for this setting. One difference is that here we use the recent `dense' approximation results of \cite{kohler_full} that enable us to be more parsimonious (if this is desired) in terms of the total number of parameters that should be sampled in the posterior. Contrary to our `dense' approach, these sparse priors require some additional hyper-priors on the sparsity level, increasing computational costs. Updating the proofs from \cite{polson2018posterior}, one could also  deploy SAS priors with denser architecture but SAS priors still  have more (hyper-)parameters overall, with the probability of being nonzero to be tuned as well as the slab parameter (the same remark applies for their variational counterparts). We also note that currently all these priors achieve adaptation through a form of hyper-prior on the architecture, which complicates the use of adaptive MCMC algorithms and may make them intractable as the width of the model grows. This challenge reinforces the need for more scalable (though only approximative) VB algorithms, where (in contrast to MCMC) adaptivity can be managed through sampling over different architectures, a process made significantly more efficient and parallelizable by recent techniques from \cite{ohn2024adaptive}. Contrary to SAS, our heavy-tailed priors enable both sampling from the full posterior with usual MCMC methods (see the paragraph on simulations below) as well as the use of mean-field variational approximation.

Two other recent natural constructions are the priors considered in 
\cite{lee2022asymptotic} and the ones in  \cite{kongkim24}. The work \cite{kongkim24} adapts, as we do, the approximation theory of \cite{kohler_full} to allow for dense (non-sparse) approximations, thus avoiding the need for forcing some weights to be zero -- as a side note, let us mention that more or less simultaneously the need for a control on the amplitude of the weights for use in Bayesian arguments has been recognised by several authors that have adapted 
Theorem 2 in \cite{kohler_full} to include quantitative bounds on network coefficients; see for example Theorem K.1 in \cite{ohn2024adaptive}, and Theorem 1 in \cite{kongkim24}  and Appendix \ref{app : bound} for our own construction --. However, an important difference of our work with the approach of \cite{kongkim24}  is that the latter requires to sample from the width of the network to achieve adaptation; as noted above, such hyper-parameter sampling can be costly for simulations, as it requires access to a good approximation of the posterior on the width (similar to the calibration of the cut-off of sieve priors in nonparametrics, see \cite{ac23} for more discussion on this); this hyper-parameter sampling is not required here, as we use a deterministic `overfitting' architecture. 

The work \cite{lee2022asymptotic} allows for heavy-tailed priors, such as mixture of Gaussians, although crucially the conditions of Theorem 3 therein are non-adaptive (i.e. the priors' parameters depend on the unknown smoothness of the regression function). A main insight of the present work, following the idea recently introduced in \cite{ac23} in nonparametric settings, is that by taking an (overfitting) width, heavy tails on weights make the (tempered) posterior automatically adapt to unknown structural parameters, let it be smoothness, intrinsic dimension(s) or compositional structures. The simulation aspect of the work \cite{lee2022asymptotic} is also relevant here, since therein the authors develop and run MCMC algorithms that enable to sample from the corresponding neural network posteriors. 

Finally, let us mention the work by \citet{ghosh2019model}, who provide algorithms to sample from variational approximations of posteriors corresponding to horseshoe priors on weight parameters. Although the authors provide no theoretical back-up for their empirical results, this is related in spirit to our approach; in fact, although there are some notable differences in the prior's choice (e.g. the horseshoe has a pole at zero and has a tuning sparsity parameter) we think that our approach and proofs could presumably be adapted in order to provide theoretical understanding and validation of their method; this idea is also  supported by current work in progress \cite{ace24},  where we derive theory for horseshoe priors for nonparametrics. \\

\subsection{Discussion: from Gaussian to Heavy-tailed weights}

As we have seen, two main families of priors for Bayesian DNNs are 1) sparsity-inducing priors and 2) priors on fully-connected architectures. 
In this section, we focus on the latter: perhaps the most natural choice of prior on the weights seems to be an isotropic Gaussian.

We discuss two ways in which heavy-tailed priors naturally emerge in this context. First, in a classical frequentist setting, where Gaussian weights are used as initializations for finding an empirical risk minimizer via gradient descent, the weights tend to become increasingly heavy-tailed during training. Second, in a Bayesian framework, the output distribution of a neural network with isotropic Gaussian weights becomes progressively more heavy-tailed as the network depth increases.

Regarding the first point: starting from an isotropic Gaussian initialization of weights, empirical evidence indicates that when using MCMC, in order to improve the posterior performance, the likelihood has to be raised to a certain power $\alpha > 1$, this is the so-called "cold-posterior effect" (\cite{wenzel2020good}) suggesting that the Gaussian prior is misspecified (in the sense that it does not correctly represent the data). When training a feedforward neural network in a frequentist manner, such as with stochastic gradient descent, it has been observed (for instance by \cite{fortuin2021bayesian}) that the weights become progressively more heavy-tailed during training. This suggests that an heavy-tailed prior would be less misspecified and indeed \cite{fortuin2021bayesian} show that the influence of the "cold-posterior effect" tends to vanish for heavy-tailed priors such as Laplace and t-Student. These priors perform well in reconstruction tasks for any power $\alpha$, whereas the Gaussian prior appears to require more tuning for this parameter. Note that we show in all our theorems that the heavy-tailed posterior contracts at minimax rate for any $\alpha <1$; the aforementioned experimental studies on "cold posteriors" suggest that for our heavy-tail prior the choice of $\alpha$ has little effect on the predictive capabilities of the estimate, contrary to its Gaussian counterpart.

Now, turning to the second point: interestingly, heavy-tailed priors can heuristically be connected to the simple isotropic Gaussian prior. Indeed, it 
 has been showed by \cite{DBLP:conf/icml/VladimirovaVMA19} that while starting from an isotropic Gaussian prior, the marginal distribution of the post and pre-activations (the input and ouput of the neurons) become more and more heavy-tailed as the depth increases, which indicates that heavier-tailed distributions are a natural occurrence in neural-network priors. We refer to the review paper of \cite{arbel2023primer} for more discussion and references on this.

In summary, this discussion reveals interesting links between isotropic Gaussian priors on weights and heavy-tailed ones; a better understanding, in particular from the theoretical point of view, is left for future work.

\subsection{Discussion: algorithms.}

{\em Algorithms.} Although an in-depth simulation study is beyond the scope of the present paper, we note that algorithms are readily available to sample from the heavy-tailed tempered posteriors introduced here. As mentioned above, since there are no structure hyper-parameters to sample from (such as the network architecture, or the position of the non-zero coefficients in the case of SAS priors),
one can use an MCMC algorithm such as the one from \cite{lee2022asymptotic}. Another possibility is the use a variational Bayes algorithm for heavy tailed priors, such as the one considered in \citet{ghosh2019model} for the horseshoe prior. 
This is left for future work. We note that having both feasible MCMC and variational algorithms makes it an interesting setting. Indeed,  
 it is known that mean-field VB, not being as complex as the posterior, may distorts certain aspects of the latter, such as posteriors for finite-dimensional functionals. On the other hand, unlike VB, the full posterior on network will retain correlations, which are believed to be particularly important for deep neural networks. It will be interesting to compare the behaviour of both algorithms through an extended simulation study; this will be considered elsewhere.

In principle  possible, hyperparameter posterior sampling (such as posterior sampling from complexity-based priors with hyper-priors on the complexity) is often believed to be delicate in terms of mixing of MCMC (a classical reference discussing this is, for example, \cite{brooksetal03}), and one may heuristically think that such hyper-sampling, on top of possibly begin time and/or ressource consuming, may add some `noise' to the MCMC output. One may argue that any prior distribution, including a prior on the complexity for instance, has a 'constant' to be tuned, but we believe that it is fair to  say that any prior with `one level of hierarchy less' is (at least philosophically)  computationally simpler.

Although MCMC sampling is made easier with a fixed architecture, it is expected that for very large networks, using current methods, MCMC sampling will become computationally intractable, the posterior being a multi-modal distribution of very high dimensionality. Recent techniques to accelerate sampling using MCMC for BNNs have been obtained by \cite{hron2022wide} and \cite{pezzetti2025functionspacemcmcbayesianwide}. As the size of the architecture grows, a practician might consider using VB to obtain a computable approximation of the posterior, perhaps at the cost of the quality of uncertainty quantification, another choice would be for example to use a Laplace approximation, see e.g. \cite{arbel2023primer} for more references on this aspect. Although our results provide theoretical support for a Heavy-tailed mean-field VB approximation it is the case that most practical algorithms are developed with Gaussian priors in mind and a study of efficient VB algorithms for more general distributions is an interesting research path. Note that in the Gaussian case, up until now adaptation is theoretically supported only when an hyper-prior is used on the architecture (e.g. as in the work of \cite{kongkim24}). In order to obtain a scalable variational result for such a prior one would have to adapt parallelization techniques, using for example the approach suggested by \cite{ohn2024adaptive}, where a variational class with uniform prior has been studied. This remains to be done to the best of our knowledge.

Another relevant aspect of the approach in practice is the possibility to make different choices for scaling sequences $(\sigma_k)$. For simplicity we have mostly focused on the constant choice as in \eqref{fixedsig}. 
Other choices may be interesting, in particular for small sample sizes, for instance taking scalings matching the lower bound in 
\eqref{condition sigma}. Indeed, since one expects a number of coefficients to be of significant amplitude, it may help to have at least a few scalings of the order of a constant, as is the case  for the latter choice. This may accelerate, in small or moderate samples, the `soft selection' of the weights of highest amplitude in the network. Again we leave investigating this for future work. \\

\subsection{Discussion: fractional posteriors $(\al<1)$ and classical posteriors $(\al=1)$} 
The main (technical) reason for which we have primarily focused on fractional posteriors is 
that proving convergence for these {\em only} requires a prior-mass control. In particular, one avoids the necessity to build sieve sets for which the complexity (e.g. entropy) needs to be controlled, as in the usual conditions using a generic theorem such as the one of \cite{GGV}. A first reason why building a sieve set is difficult here is that we use a prior that `overfits', that is, all weights in the neural network are modelled, so one cannot take as sieve the set of neural networks built on the optimal `oracle' architecture (one that gives the optimal convergence rate). A second reason is connected to heavy-tails: the usual contraction Theorem in \cite{GGV} requires exponentially fast decrease of the prior mass of the complement of the sieve sets. But working with heavy-tailed priors on the weights makes verifying this difficult, unless the sieve sets are very large. 

To illustrate this discussion, it is helpful to consider the example of priors on functions defined as random series on a given orthonormal basis. In the case of normally distributed coefficients, the work \cite{vvvz08} shows optimal posterior contraction following the generic approach of  \cite{GGV}; this has been extended to $p$-exponential priors in \cite{adh21} ($p\in[1,2]$, the case $p=1$ corresponding to Laplace priors on coefficients, the case $p=2$ recovering the Gaussian case), their proofs strongly relying on exponential-type decrease of complements of sieve sets. The case of heavy-tailed coefficients has been recently investigated in \cite{ac23};  therein, it is proved that the (classical) posterior distribution for a heavy-tailed series prior indeed converges at a near-optimal adaptive minimax rate for a very specific model, namely the Gaussian white noise: although one cannot apply the generic contraction theorem, for that model one is able to prove convergence by a direct analysis of the posterior. For more complex models (such as random design regression or classification), proving this is an open problem. We refer to the discussion in the Appendix of  \cite{ac23} for more technical details on how one could attempt to prove a posterior rate for the classical posterior in general settings. 

In Section \ref{sec : truepost},  we provide a result for the classical posterior $(\al=1)$ for the augmented prior that also models the noise variance. By exploiting a link between classical and fractional posterior in this case, one is able to circumvent the use of sieve sets and control of the complexity of these. 

Coming back to fractional posteriors $(\al<1)$, although our results do not enable to discriminate between different choices of $\alpha$, this choice is an important practical question. More generally, the question of how to choose $\alpha$ for fractional posteriors goes much beyond the scope of the present contribution, but we give now a few pointers to recent work on this point. One motivation for the use of $\alpha$--posteriors for statistical inference is their greater robustness to model misspecification compared to the usual Bayesian posterior, as discussed for instance in works by  \cite{Grunwald_2012} and by \cite{bhw16}. The latter work also discusses possible ways to choose $\al$. Other works in this direction include \cite{hw17, lhw19, sm19}. 
Although beyond the scope of the paper, extending the results to possibly misspecified models with additional insight on the choice of $\alpha$ is an interesting research direction. \\

\subsection{Discussion: possible extensions with geometric data} 

In Section \ref{sec:geom}, we work under Assumption \ref{hyp : mink}, which assumes that the covariates \((X_i)\) are supported on a compact set with bounded Minkowski dimension. A natural extension of this model would involve adding noise to these covariates. For example, one could assume that observations are sampled from points on a compact manifold with some additive noise locally perpendicular to the manifold. Alternatively, one could assume the covariates are supported within a tubular region around a compact manifold.

Another interesting but more complex model would combine both anisotropy (as in Section \ref{sec:anis}) and manifold assumptions by, for instance, allowing the true function to exhibit different levels of smoothness along and orthogonal to the manifold. Although this is beyond the scope of the present paper, such a model has been studied within a nonparametric Bayesian framework (with simpler priors, not neural network ones) in density estimation by \cite{berenfeld2024estimatingdensitynearunknown}.

We now discuss what happens in the case of a tubular region. Assume that there exists a compact manifold $M \subset [0,1]^d$ with dimension $d^* < d$ and a ``noise level" $\delta > 0$ such that the distribution of the covariates $P_X$ is supported on the $\delta$-tubular region around $M$,

\[ M_{\delta} := \{ x \in [0,1]^d \, : \, d(x,M) \leq \delta \}. \]

The first step for applying our analysis (as in Theorem \ref{thm : mink} for instance) would be to establish an approximation result for \( f_0 \) on \( M_{\delta} \) using wide DNNs. Since this assumption is strictly weaker than Assumption \ref{hyp : mink}, one would expect a slower approximation rate that depends on both \( d^* \) and \( \delta \). In fact, when \( \delta \) is too large, minimax approximation results may not be feasible due to the increased difficulty of the problem. For a Hölder-smooth \( f_0 \in \mathcal{C}^{\beta}(F) \), one such approximation result is given in \cite{jiao2023deep}. By following the proof of their Theorem 6.1 and applying our usual enlargement of the network (for instance as in Proposition \ref{lemma : approxcomp}), one can show that there exists a DNN realization \( \tilde{f} \), with depth \( \log^2 n \) and width \( \sqrt{n} \), such that as \( n \to \infty \),
\[
\| \tilde{f} - f_0 \|_{L^{2}(P_X)} \lesssim n^{- \frac{\beta}{2 \beta + d_{\eta}}},
\]
where \( d_{\eta} = O(d^* \log d) \), provided that the tubular region is not too large in the sense that \( \delta \leq C n^{- \frac{\beta}{2 \beta + d_{\eta}}} (\log n)^{-4 \beta / d_{\eta}} \), with a constant \( C \) explicitly derived by \cite{jiao2023deep}. Notably, this rate includes a factor of \( d^* \log d \) rather than \( d^* \) as in the exact manifold case, resulting in a slightly slower rate.

A second step would involve verifying that the weights of this approximating neural network are bounded polynomially in terms of the sample size. This could be achieved by following the proofs in \cite{jiao2023deep} to track these bounds, in a similar way as we did for the results by Kohler and Langer in \cite{kohler_full} in Appendix \ref{app : bound}.

\subsection{Further work and open questions}  

Let us mention a few further natural questions arising from the present results beyond the just-mentioned aspects on simulations. 
First, although we are able to derive a result for the classical posterior with an augmented prior, the question is left open as to whether posterior contraction still holds {\em without} putting a prior on the noise variance. In the setting of random series priors,  
 the simulation study in \cite{ac23} conducted in regression (as well as density and classification models) for both fractional and classical posteriors suggests that there is no visible phase transition when $\al$ increases from e.g. $1/2$ up until $1$ (included); we conjecture that all the theoretical results presented here (with no prior on $\ta^2$) still go through for standard posteriors ($\alpha=1$).
 
Also, as mentioned above, a promising direction is to derive results for horseshoe priors, for instance in the spirit of \cite{ghosh2019model}, as well as to compare with the present heavy-tailed priors. We are currently investigating this in the simpler setting of white noise regression with series priors in \cite{ace24}. 

Finally, it would be interesting to understand more the effect of overfitting with our approach: as we have noted, we can choose a polynomial width with large power for the network; also, for the specific choice of constant $\sigma_k$ the prior is exchangeable and does not particularly penalise a high number of large coefficients; one may then think that this setting may be particularly appropriate to test for the presence of a {\em double descent} phenomenon, which features a decrease of the risk for very overfitted models. It is conceivable that for very large widths one sees this appear in the present context: this deserves further investigation.

\section{Proofs of the main results} \label{sec:proofs}

In this section we give the proofs of the main results for the compositional models.

\subsection{Fully connected ReLU DNN approximation}\label{subsec : approx}
We recall here the approximation result we will use to study posterior contraction towards compositional functions as in Theorems \ref{thm : comp} and \ref{thm : varcomp}.

Proofs are done in appendix \ref{app : add proofs}. For Proposition \ref{lemma : approxcomp} we follow the same approach as \citet[Theorem 1]{JSH}, simply using the fully-connected architecture from \citet{kohler_full} for the elementary building blocks, as it is better suited to our prior. Note that \citet{kohler_full} does not give a explicit bound on the magnitude of the coefficients in the approximating neural network $\Tilde{f}$. Such a bound is fundamental in the study of the properties of our prior and is provided in appendix \ref{app : bound}.

\begin{proposition}\label{lemma : approxcomp}
    Let $f \in \mathcal{G}(q, \mathbf{d}, \mathbf{t}, \boldsymbol{\beta}, K)$, $L$ as in \eqref{logdepth} and $\phi_n$ as in \eqref{ratephi}. There exists a ReLU network
    \[ \Tilde{f} \in \mathcal{F}(L,\br(\sqrt{n}),s)\]
    with 
    \[ s \lesssim n \phi_n^2 \log^{1-\gamma} n\]
    such that for sufficiently large $n$,
    \[|| f - \Tilde{f}||_{\infty} \lesssim \phi_n.\]
    Additionally, all coefficients of $\Tilde{f}$ satisfy $|\Tilde{\theta}_k| \leq n^{c_{\beta}} $, where $c_{\beta} \geq 1$ is a constant only depending on $(\beta_i)_i$.
\end{proposition}

\begin{remark}\label{rmq : seuil}
We can be more precise in Proposition \ref{lemma : approxcomp} and construct the network $\Tilde{f}$ in such a way that the position of the active coefficients is known. From the proof given in Appendix \ref{app : approx proofs} it follows that there is an integer \[ r^* \asymp \underset{0 \leq i \leq q}{\max} \left( \frac{n}{\log^{\gamma}n} \right)^{\frac12 \frac{t_i}{2 \beta_i^* + t_i}}, \]
such that if $(\Tilde{W}_l,\Tilde{v}_l)$ are the coefficients of $\Tilde{f}$, we have for all $l \in [L+1]$, $\Tilde{W}_l^{(ij)} =0$ whenever $ i \vee j > r^*$ and $\Tilde{v}_k^{(i)} =0$ whenever $i > r^* $.
\end{remark}

\subsection{Proof of Theorem \ref{thm : comp}}\label{proof : thmcomp1}
\begin{proof} In view of Lemma \ref{lem : conc}, it suffices to show that there exists $C > 0$ such that for sufficiently large $n$,
\[ \Pi\left( \lVert f - f_0 \rVert_{\infty} \leq  \phi_n \right) \geq e^{- C n  \phi_n^2 }.\]

Applying Proposition \ref{lemma : approxcomp}, we obtain $\Tilde{f_0} \in \mathcal{F}(L,\br(\sqrt{n}),s)$ a ReLU neural network that approximates $f_0$ for large enough $n$ with $\lVert f_0 - \Tilde{f_0} \rVert_{\infty} \leq \phi_n/2$ and $s \lesssim n \phi_n^2 \log^{1-\gamma} n$. Using the triangle inequality,
\[\Pi\left( \lVert f - f_0 \rVert_{\infty} \leq \phi_n \right) \geq \Pi\left( \lVert f - \Tilde{f_0}  \rVert_{\infty} \leq \phi_n/2 \right).\]
It suffice to control the difference in the supremum norm between two networks of the same structure. For this purpose we use Lemma \ref{propag1}. Let $\theta_k$ (resp. $\Tilde{\theta}_{k}$) be the coefficients of $f$ (resp. $\Tilde{f}_0)$. We know from Proposition \ref{lemma : approxcomp} that $\sup_k |\Tilde{\theta}_{k}| \leq n^{c_{\beta}}$ where $c_{\beta} \geq 1$. Recall that $V = \prod_{l=0}^L (r_l+1)$ thus using lemma $\ref{propag1}$ and independence, 

\begin{equation}\label{eq : beforesplit}
    \Pi\left( \lVert f - \Tilde{f_0}  \rVert_{\infty} \leq \phi_n /{2} \right) \geq \prod_{k=1}^T \Pi \left( |\Tilde{\theta}_k -\theta_k | \leq \frac{\phi_n}{2 n^{c_{\beta} L}V (L+1)} \, , \, |\theta_k| \leq n^{c_{\beta}} \right).
\end{equation}
From remark \ref{rmq : seuil} there is an integer
\begin{equation}\label{rappel : r*}
r^* \asymp \underset{0 \leq i \leq q}{\max} \left( \frac{n}{\log^{\gamma}n} \right)^{\frac12 \frac{t_i}{2 \beta_i^* + t_i}}
\end{equation}
such that, for all $l \in [L+1]$, $\Tilde{W}_l^{(ij)} =0 $ whenever $i \vee j > r^*$. Let \begin{equation}\label{def : z_n}
     z_n := \phi_n/(2 n^{c_{\beta} L} V (L+1)),
\end{equation}
note that $z_n \to 0$ when $n \to \infty$, so that $0 < z_n < 1$ when $n$ is large enough. 
 
We can then split the right hand side in \eqref{eq : beforesplit} and first take care of the possibly large coefficients (but still bounded by $n^{c_{\beta}})$ using the heavy tails properties of the prior. To do so we use Lemma \ref{lem : coeff grand} with $z_n$, on every coordinate such that $i \vee j \leq r^*$. We get a large enough constant $C_0 > 0$ such that when $n$ is sufficiently large,
\begin{align*} 
 \prod_{l=1}^{L+1} \prod_{i \vee j \leq r^*} \Pi & \left(  | W_{l}^{(ij)} -  \Tilde{W}_{l}^{(ij)} | \leq z_n \, , \, |W_{l}^{(ij)}| \leq n^{c_{\beta}} \right) \\
 & \geq \prod_{l=1}^{L+1} \prod_{i \vee j \leq r^*} \underbrace{z_n \exp\left({- C_0 \log^{1+\kappa}\left( (1 + n^{c_{\beta}})({\sigma_l^{(ij)}})^{-1} \right)} \right)}_{ := A}.
\end{align*}
Recall that $V = \prod_{l=0}^L (r_l+1) = (d+1)(\lceil \sqrt{n} \rceil +1)^{L}$ and $L = \lceil \log^{1 + \delta} n \rceil$, thus there is $C'$ a large enough constant such that, when $n$ is sufficiently large,
    \begin{equation}\label{boundV}
       2n^{c_{\beta} L}(L+1)V \leq e^{C' \log^{2 + \delta} n}.
    \end{equation}  
    Using $\log\left(({\sigma_l^{(ij)}})^{-1}\right) \leq \log^{2(1+\delta)}n$ from \eqref{condition sigma}, this leads to the bound, for $C_3$ a large enough constant
    \begin{equation}\label{minoration A}
    A \geq \phi_n e^{- C' \log^{2 + \delta} n - C_0 \log^{1 + \kappa}\left((1 + n^{c_{\beta}})({\sigma_l^{(ij)}})^{-1} \right)} \geq \phi_n e^{-C_3 \log^{2(1+\delta)(1+\kappa)}n}.
    \end{equation}
Now since $\phi_n \asymp (n / \log^{\gamma}n)^{- \nu} $ for some $\nu > 0$, there exists a constant $C_4 > 0$ such that
    \[\prod_{l=1}^{L+1} \prod_{i \vee j \leq r^*} A  \geq e^{-  C_{4} (r^*)^2 L \log^{\gamma -1}n}. \]

Finally since $(r^*)^2L \leq s \lesssim n \phi_n^2 \log^{1-\gamma}n$, we get for the indices such that $i \vee j \leq r^*$, a large enough constant $C>0$ such that 
\[ \prod_{l=1}^{L+1} \prod_{i \vee j \leq r^*} \Pi \left(  | W_{l}^{(ij)} - \Tilde{W}_{l}^{(ij)} | \leq z_n \, , \, |W_{l}^{(ij)}| \leq n^{c_{\beta}} \right) \geq e^{-C n \phi_n^2}.\]

Now we take care of the small coefficients using the decay $(\sigma_k)$. For the indices such that $i \vee j > r^*$, recall that we have $\Tilde{W}_l^{(ij)} = 0$ for all $l$, leaving us to bound
\[ \prod_{l=1}^{L+1} \prod_{i \vee j > r^*} \Pi \left(  | W_{l}^{(ij)} | \leq z_n \right) = \prod_{l=1}^{L+1} \prod_{i \vee j > r^*}  \left[ 1 - 2 \overline{H} \left( z_n /  \sigma_l^{(ij)} \right)  \right] ,\]
here we have used the symmetry of $h$ and the fact that $z_n < n^{c_{\beta}}$ for $n$ sufficiently large. Using $\log^{2(1+ \delta)} ( i \vee j ) \leq \log(1/ \sigma_l^{(ij)} )$ from \eqref{condition sigma} and the expression of $r^*$ given by \eqref{rappel : r*}, we have a large enough constant $C_5$ such that whenever $i \vee j > r^*$, we have
\[ 1/ \sigma_l^{(ij)} \geq e^{\log^{2(1+\delta)}(i \vee j)} \geq e^{ \log^{2(1+\delta)}r^*} \geq e^{ C_5 \log^{2(1+\delta)}n}.\]

Recall that by assumption \ref{H3}, $\overline{H}(x) \leq c_2 /x$ for all $x \geq 1$. Noting that $\log(1-2x) \geq -4x$ whenever $x < 1/4$ and using the bound \eqref{boundV} on $V$, one gets that for sufficiently large $n$ satisfying 
\[  z_n e^{C_5 \log^{2(1 + \delta)}n} \geq 1 \vee 4c_2 ,\]
there is a constant $C_6$ such that
\[  1 - 2 \overline{H} \left( z_n /  \sigma_l^{(ij)} \right)  \geq 1 - 2c_2 z_n^{-1}e^{-C_5 \log^{2(1 + \delta)}n} \geq \exp(-C_6 z_n^{-1} e^{-C_5 \log^{2(1 + \delta)}n}). \]
Recall $T$ is the total number of parameters of the network, this leads to the lower bound
\[ \prod_{l=1}^{L+1} \prod_{i \vee j > r^*} \Pi \left(  | W_{l}^{(ij)} | \leq z_n \right)  \geq \exp(-C_6 T z_n^{-1} e^{-C_5 \log^{2(1 + \delta)}n}). \]
Whenever $n$ is large enough and for our choice of architecture, we have $T \lesssim n \log^{1 + \delta}n$, using again the bound \eqref{boundV} on V one gets for $n $ sufficiently large, \[T n^{c_{\beta}L}V (L+1) \lesssim n ( \log^{(1+\delta)}n ) \,e^{C'\log^{2+\delta}n} \lesssim n \phi_n^3 e^{ C_5 \log^{2(1+\delta)} n} .\]
Finally we have a large enough constant $C>0$ such that 
\[ \prod_{l=1}^{L+1} \prod_{i \vee j > r^*} \Pi \left(  | W_{l}^{(ij)} | \leq z_n \right) \geq e^{-C n \phi_n^2}.\]

\end{proof}

\subsection{Proof of Theorem \ref{thm : varcomp}}\label{proof : varcomp}
\begin{proof}
We give the proof for $\kappa = 0$, giving the fastest rate, the proof in the general case being similar. In this setting $\gamma = 2(1+\delta) +1$. Let $\Tilde{f}_0 \in \mathcal{F}(L,\br(\sqrt{n}),s)$ be the network approximating $f_0$ from Proposition \ref{lemma : approxcomp}. Using $ \lVert f - f_0 \rVert_{L^2(P_X)}^2 \leq  2 \lVert f - \tilde{f}_0 \rVert_{L^2(P_X)}^2 + 2 \lVert \tilde{f}_0 - f_0 \rVert_{L^2(P_X)}^2 $, Lemma \ref{lem : pac} gives
\begin{align*}
    E_{f_0} \left( \int D_{\alpha}(f,f_0)  \, d \hat{Q}_{\alpha}(f) \right) & \leq \frac{1}{1 - \alpha}\lVert \tilde{f}_0 - f_0 \rVert_{L^2(P_X)}^2  \\ &
    + \underset{Q \in \mathcal S}{\inf} \left\{ \frac{\alpha}{1-\alpha} \int \lVert f-\Tilde{f}_0 \rVert_{L^2(P_X)}^2 \, d Q(f) + \frac{\KL(Q,\Pi)}{n(1-\alpha)}\right\}.
\end{align*}
Since $\lVert \tilde{f}_0 - f_0 \rVert_{L^2(P_X)}^2 \leq \phi_n^2$ when $n$ is large enough, it is enough to show that $\mathcal{S}_{HT}(h)$ contains a distribution $Q^*$ satisfying the extended prior mass condition (see Remark \ref{rmq : markov}), 
\begin{equation}\label{extpriormass}
    \int \lVert f-\Tilde{f}_0 \rVert_{L^2(P_X)}^2 \, d Q^*(f) \leq \phi_n^2  \qquad \text{and} \qquad  \KL(Q^*,\Pi) \leq n \phi_n^2.
\end{equation}
For $h$ the heavy-tailed density taken as prior on coefficients, let us define $Q^* :=\bigotimes_{k=1}^T Q_k^* \in \mathcal{S}_{HT}(h) $ by setting for all $k \in [T],$ \[ \frac{dQ_k^*}{d \theta_k}(\theta_k):= \frac{1}{\sigma_k} h \left( \frac{ \theta_k -\Tilde{\theta}_k}{\sigma_k}\right) = h_{\Tilde{\theta}_k,\sigma_k}(\theta_k), \]
where $\sigma_k = e^{-\log^{2(1+\delta)}n}$ and $\{\Tilde{\theta}_k \}$ are the coefficients of $\Tilde{f}_0$, such that $Q_k^* \in \mathcal{H}(h) $.

Using Lemma \ref{propag1} and setting $w_n := n^{2c_{\beta}L} V^2(L+1)^2 $, we get
\begin{align*}
    \int \lVert f-\Tilde{f}_0 \rVert_{L^2(P_X)}^2 \, d Q^*(f) \leq \int \lVert f-\Tilde{f}_0 \rVert_{\infty}^2 \, d Q^*(f) \leq w_n \int \underset{k }{\max}|\theta_k - \Tilde{\theta}_k|^2 \, d Q^*(f).
\end{align*}
Recall the definition \eqref{def : moment} of the moments of $h$, from the definition of $Q^*$ it follows,
\[\int \underset{k }{\max}|\theta_k - \Tilde{\theta}_k|^2 \, d Q^*(f) \leq \sum_{k =1}^T \int |\theta_k - \Tilde{\theta}_k|^2 d Q_k^*(\theta_k) = m_2(h) \sum_{k =1}^T \sigma_k^2  = m_2(h) T e^{-2 \log^{2(1+\delta)}n} . \]
Using inequality \eqref{boundV}, we get
$w_n T \leq e^{C \log^{2+\delta}n}$ for $C>0$ sufficiently large. Since $m_2(h) \leq c_3^{2/(2 \vee (1 + \kappa))} $ by assumption, we have $m_2(h) w_n  T \sigma_k^2   \leq \phi_n^2$ if $n$ is large enough.

\medskip

We just showed that $Q^*$ satisfies the first condition in \eqref{extpriormass}, let us check the second one. We can write $\Pi = \bigotimes_{k=1}^T \Pi_k$ where,
\[ \frac{d \Pi_k}{d \theta_k} (\theta_k) = \frac{1}{\sigma_k} h \left( \frac{\theta_k}{\sigma_k}\right) = h_{0,\sigma_k}(\theta_k).\]

By additivity using independence we have $\KL(Q^*,\Pi) = \sum_{k=1}^T \KL(Q_k^*,\Pi_k) $. Let us write $S_0 = \{ k \in [T]  \, , \, \Tilde{\theta}_k \neq 0 \}$ the support of the collection of coefficients of $\Tilde{f}_0$, since $Q_k^* = \Pi_k$ whenever $\Tilde{\theta}_k =0$, we get 
\[ \KL(Q^*,\Pi) = \sum_{k \in S_0} \KL(Q_k^*,\Pi_k). \]
Let $k \in S_0$. Using the fact that $h$ is bounded,
\begin{align*}
    \KL(Q_k^*,\Pi_k) &= \int \log \left( \frac{h \left( \frac{ \theta - \Tilde{\theta}_k}{\sigma_k}\right)}{h \left( \frac{ \theta }{\sigma_k}\right)} \right)  \frac{1}{\sigma_k} h \left( \frac{ \theta - \Tilde{\theta}_k}{\sigma_k}\right) \, d\theta \\
    &\leq \log ||h||_{\infty} + \underbrace{\int \log \left( \frac{1}{h \left( \frac{ \theta }{\sigma_k}\right)} \right)  \frac{1}{\sigma_k} h \left( \frac{ \theta - \Tilde{\theta}_k}{\sigma_k}\right) \, d\theta}_{ I_k}.
\end{align*}
Using the heavy-tailed assumption \ref{H2} with $\kappa = 0$ on $h$, for all $\theta >0$, we get
\[ \log  \left( \frac{1}{h(\theta \sigma_k^{-1})} \right)
\leq c_1 (1+\log(1+ \theta \sigma_k^{-1})) .
\]
Splitting the integral $I_k$ in two and using the symmetry of $h$ leads to
\begin{align*}
    I_k &\leq \int_0^{+\infty} \log  \left( \frac{1}{h(\theta \sigma_k^{-1})} \right) \frac{1}{\sigma_k} h \left( \frac{ \theta - \Tilde{\theta}_k}{\sigma_k}\right) \, d\theta + \int_0^{+\infty} \log  \left( \frac{1}{h(\theta \sigma_k^{-1})} \right) \frac{1}{\sigma_k} h \left( \frac{ \theta + \Tilde{\theta}_k}{\sigma_k}\right) \, d\theta \\
    &\leq \underbrace{\int_0^{+\infty} c_1 \left[1+\log(1+ \theta \sigma_k^{-1}) \right] \frac{1}{\sigma_k} h \left( \frac{ \theta - \Tilde{\theta}_k}{\sigma_k}\right) \, d\theta}_{I_k^-} + \underbrace{\int_0^{+\infty} c_1 \left[1+\log(1+ \theta \sigma_k^{-1}) \right] \frac{1}{\sigma_k} h \left( \frac{ \theta + \Tilde{\theta}_k}{\sigma_k}\right) \, d\theta}_{I_k^+}.
\end{align*}
One can suppose $\Tilde{\theta}_k > 0 $, given the symmetry in $I_k^+$ and $I_k^-$. Considering first $I_k^-$, we have, by changing variables, 
\[ I_k^- \leq c_1 + c_1 \int_{- \Tilde{\theta}_k/\sigma_k}^{+\infty} \log \left(1+ \frac{\Tilde{\theta}_k}{\sigma_k} +u \right) h \left( u \right) \, du.\]
The integral in the last display is bounded as follows
\begin{align*}
    \int_{- \Tilde{\theta}_k/\sigma_k}^{+\infty} \log \left(1+ \frac{\Tilde{\theta}_k}{\sigma_k} +u \right) h \left( u \right) \, du &\leq \log \left(1+ 2\frac{\Tilde{\theta}_k}{\sigma_k}  \right) + \int_{\Tilde{\theta}_k/\sigma_k}^{+\infty} \log \left(1+ 2u \right) h \left( u \right) \, du \\
    &\leq \log \left(1+ 2\frac{\Tilde{\theta}_k}{\sigma_k}  \right) + 2 m_1(h).
\end{align*}
e bound can be obtained for $I^+$, and using the fact that $|\Tilde{\theta}_k| \leq 1$ in the structure of corollary \ref{corJSH} we get,
\[ I \lesssim \log(\sigma_t^{-1}),\]
For $I_k^+$, we have, changing variables

\begin{align*}
    \int_{ \Tilde{\theta}_k/\sigma_k}^{+\infty} \log \left(1 -\frac{\Tilde{\theta}_k}{\sigma_k} +u \right) h \left( u \right) \, du &\leq \int_{ \Tilde{\theta}_k/\sigma_k}^{+\infty} \log \left(1 +u \right) h \left( u \right) \, du \leq m_1(h).
\end{align*}
In the case of $\kappa > 0$ one gets $m_{1+ \kappa}(h)$ instead of the first order moment, this quantity is also bounded by assumption.
Finally, we have $I_k \lesssim 1 + \log \left(1+ 2\frac{|\Tilde{\theta}_k|}{\sigma_k}  \right)$ for $n$ large enough. Using the fact that $|\Tilde{\theta}_k| \leq n^{c_\beta}$ we obtain

\[ \KL(Q_k^*,\Pi_k) \lesssim 1 + \log \left(1+ 2\frac{|\Tilde{\theta}_k|}{\sigma_k}  \right) \lesssim \log^{2(1+\delta)}n = \log^{\gamma -1}n. \]
Recalling that $|S_0| = s \lesssim  n \phi_n^2 \log^{1-\gamma}n$, we finally get,
\[ \KL(Q^*,\Pi) \lesssim s \log^{\gamma -1}n \lesssim n \phi_n^2 ,\]
which concludes the proof.
\end{proof}

\acks{
The authors would like to thank Sergios Agapiou and Gabriel de la Harpe for helpful discussions. IC acknowledges funding from the Institut Universitaire de France and ANR grant project BACKUP ANR-23-CE40-0018-01. 
}

\appendix
\section{Tempered posterior contraction}\label{app: post}

In this appendix we recall some results regarding the concentration of posterior distributions.
\begin{definition}\label{def : Rényi}
    Let $\alpha \in (0,1)$ and $P,Q$ be two probability measures. The Kullback-Liebler (KL) divergence between $P$ and $Q$ is defined by
    \begin{equation*}
        \KL (P,Q) := \int \log \left( \frac{dP}{dQ} \right) \, dP \quad \text{if } P \ll Q, \text{ and } +\infty \text{ otherwise.}
    \end{equation*}
    Let $\mu$ be any measure satisfying $P \ll \mu$ and $Q \ll \mu$. The $\alpha$-Rényi divergence between $P$ and $Q$ is defined by
    \begin{equation*}
        D_{\alpha}(P,Q) := \frac{1}{\alpha - 1} \log \left( \int \left( \frac{dP}{d\mu}\right)^{\alpha} \left( \frac{dQ}{d\mu}\right)^{1-\alpha} \, d\mu \right) .
    \end{equation*}
\end{definition}
For an overview on properties of such divergences we refer to \citet{vanerven}. In particular, if $P_{f_0,\tau_0^2}$ is the probability measure such that $(X_i,Y_i) \overset{i.i.d.}{\sim} P_{f_0,\tau_0^2}$ in the regression setting \eqref{model}, using additivity of the Rényi divergence for independent observations, for any $f \in \mathcal{F}$ and $\tau^2 > 0$, we have
\begin{equation}\label{def : renyi}
    D_{\alpha}(P_{f,\tau^2},P_{f_0,\tau_0^2}) = \frac1n D_{\alpha}(P_{f,\tau^2}^{\otimes n},P_{f_0,\tau^2_0}^{\otimes n}).
\end{equation}
When $\tau_0$ is supposed to be known (e.g. $\tau_0 = 1$) we simply denote
\[ D_{\alpha}(f,f_0) := D_{\alpha}(P_{f,\tau_0^2},P_{f_0,\tau_0^2}).\]
\subsection{Contraction Lemmas : case of known variance $\tau_0 = 1$.}
In order to simplify the reading of this section we first state the results we use to prove our Theorems in the case where the noise variance $\tau_0 =1$ is known. A more general version used in the unknown variance case can be found below.

\begin{lemma}\label{lem : conc}
    Let $f_0 \in \mathcal{F}$ and assume $\tau_0 =1$ is known. Let $\Pi$ be a probability measure on $\mathcal{F}$. For any $\alpha \in (0,1)$, any positive sequence $(\varepsilon_n)_n$ such that $\varepsilon_n \to 0$, $n \varepsilon_n^2 \to \infty$ and 
    \[ \Pi \left[ \{ f \, : \, \lVert f-f_0 \rVert_{\infty} \leq \varepsilon_n \} \right] \geq e^{-n\varepsilon_n^2}, \]
    there is a constant $C >0 $ such that as $n$ tends to infinity,
    \[ E_{{f_0}}\Pi_{\alpha} \left[\left \{f \, : \,  D_{\alpha}(f,f_0) \geq C \frac{\alpha \varepsilon_n^2}{1 - \alpha} \right\} \, | \, X,Y \right] \to 0 ,\]
    where $E_{f_0}$ denotes the expectation under $P_{f_0}$.
\end{lemma}
This is Theorem 4.1 in \cite{l2023semiparametric} written in our regression setting, noting that the KL--neighborhood therein simplifies to an $L^{\infty}$--neighborhood via Lemma \ref{lem : voisKL}. The version of this result that accounts for possible unknown noise variance $\tau_0$ is Lemma \ref{lem : conc unknown var} below.

\begin{lemma}\label{lem : pac}
    Let $f_0 \in \mathcal{F}$ and assume $\tau_0 =1$ is known. Let $\Pi$ a probability measures on $\mathcal{F}$ and $\mathcal{S}$ a set of probability measure on $\mathcal{F}$. For $\alpha \in (0,1)$ let $\hat{Q}_{\alpha}$ be the tempered variational approximation on $\mathcal{S}$ defined by \eqref{def : tvp}. Then
    \begin{equation*}
    E_{f_0} \left( \int D_{\alpha}(f,f_0)  \, d \hat{Q}_{\alpha} (f) \right) \leq \underset{Q \in \mathcal S}{\inf} \left\{ \frac{\alpha}{2(1-\alpha)} \int \lVert f-f_0 \rVert_{L^2(P_X)}^2 \, d Q(f) + \frac{\KL(Q,\Pi)}{n(1-\alpha)}\right\}.
  \end{equation*}
\end{lemma}
This is Theorem 2.6 from \citet{alquier2020} written in our regression framework, using \eqref{eq : kl} with $\tau_0 = \tau =1$, to express the Kullback divergence therein as an $L^2(P_X)$-squared norm.
\begin{remark}\label{rmq : markov}
    If there is $\varepsilon_n > 0$ such that 
    \[E_{f_0} \left( \int D_{\alpha}(f,f_0)  \, d \hat{Q}_{\alpha} (f) \right) \leq \varepsilon_n^2,\]
    then by Markov's inequality, for any $M_n \to \infty$, as $n \to \infty$,
    \[ E_{{f_0}} \hat{Q}_{\alpha} \left[\left \{f \, : \,  D_{\alpha}(f,f_0) \geq M_n \varepsilon_n^2 \right\}\right] \to 0 .\]
    Moreover, if we consider the clipped posterior $\Hat{Q}_\al^B$ as in Corollary \ref{cor : clip} we obtain convergence of the variational posterior mean in $L^2(P_X)$-distance, as then Lemma \ref{lem : clip} gives
    \[ E_{f_0} \big{\lVert} \int f \, d \Hat{Q}_{\al}^B(f) - f_0 \big{\rVert}_{L^2(P_X)} \leq E_{f_0}\int || f - f_0 ||_{L^2(P_X)} \, d \hat{Q}_{\al}^B(f) \lesssim \varepsilon_n .\]
    \end{remark}

\subsection{Extended Lemmas : unknown variance case}
We provide now a sufficient condition for tempered posterior contraction, recalled in Lemma \ref{lem : conc unknown var} below, it involves putting enough prior mass on the Kullback-neighborhood
    \begin{equation}\label{def : voisKL}
        B_n((f_0,\tau_0^2),\varepsilon_n) := \{ (f,\tau^2) \, : \, \KL (P_{f_0,\tau_0^2},P_{f,\tau^2}) \leq \varepsilon_n^2, \,V_2(P_{f_0,\tau_0^2},P_{f,\tau^2}) \leq \varepsilon_n^2\},
    \end{equation}
where
\[ V_2(P_{f_0,\tau_0^2},P_{f,\tau^2}) :=\int  \left( \log \left( \frac{dP_{f_0,\tau_0^2}}{dP_{f,\tau^2}} \right) - \KL(P_{f_0,\tau_0^2},P_{f,\tau^2}) \right)^2 \, dP_{f_0,\tau_0^2} .\]
In the regression setting, simple calculations allow us to identify the KL-neighborhood \eqref{def : voisKL} as an $L^\infty$-type ball.

\begin{lemma}\label{lem : voisKL} One can find a constant $C >0$ sufficiently large, such that it holds, for any $\varepsilon_n \to 0$ as $n \to \infty$,
\[ \{ (f,\tau^2) \, : \, ||f-f_0||_{\infty} \leq \varepsilon_n , \, |\tau^2 - \tau_0^2|\leq \varepsilon_n^2 \} \subset B_n((f_0,\tau_0^2),C\varepsilon_n).\]
In particular when $\tau_0 =1$ is known,
\[ \{ f \, : \, ||f-f_0||_{\infty} \leq \varepsilon_n \} \subset B_n(f_0, C\varepsilon_n). \]
    
\end{lemma}
\begin{proof} Simple calculations (see for instance Lemma 19 in \cite{cr24}) lead to
    \begin{equation*}\label{eq : kl}
        \KL(P_{f_0,\tau_0}, P_{f,\tau}) = \frac{1}{2 \tau^2}||f-f_0||_{L^2(P_X)}^2 + \frac{1}{2}\left( \log(\frac{\tau^2}{\tau_0^2}) + \frac{\tau_0^2}{\tau^2} - 1 \right) ,
    \end{equation*}
    \begin{equation*}\label{eq : klvar}
        V_2 (P_{f_0,\tau_0}, P_{f,\tau}) = \frac{\tau_0^2}{\tau^4} ||f-f_0||_{L^2(P_X)}^2 + \frac{(\tau_0^2 - \tau^2)^2}{2 \tau^4} + \frac{1}{4\tau^4}\left(||f_0-f||_{L^4(P_X)}^4 - ||f_0-f||_{L^2(P_X)}^4 \right).
    \end{equation*}
Using for any $p\geq 1$, $||f_0-f||_{L^p(P_X)}\leq||f-f_0||_{\infty}\leq \varepsilon_n$ and $|\tau^2 - \tau_0^2|\leq \varepsilon_n^2$ as well as $\varepsilon_n \to 0$ when $n \to \infty$, we get
\begin{align*}
        \KL(P_{f_0,\tau_0^2}, P_{f,\tau^2}) &\leq \frac{\varepsilon_n^2}{\tau_0^2-\varepsilon_n^2} + \frac12 \log(1 + \frac{\varepsilon_n^2}{\tau_0^2}) \leq C \varepsilon_n^2 \\
        V_2 (P_{f_0,\tau_0^2}, P_{f,\tau^2}) &\leq \frac{\tau_0^2 \varepsilon_n^2}{(\tau_0^2-\varepsilon_n^2)^2} + \frac{3\varepsilon_n^4}{4(\tau_0^2-\varepsilon_n^2)^2}  \leq C \varepsilon_n^2.
    \end{align*}
\end{proof}
The contraction Theorems for tempered posteriors in Section \ref{sec:main} are obtained applying Lemmas \ref{lem : conc} and \ref{lem : pac} above (or Lemma \ref{lem : conc unknown var} in the case of unknown variance), therefore the rates are formulated in terms of Rényi divergences. In the random design Gaussian regression model \eqref{model}, one can relate such rates to usual $L^2(P_X)$ ones, provided the true function $f_0$ has a known upper bound $M_0$.
\begin{lemma}\label{lem : clip}
    Let $P_{f,\tau^2},P_{f_0,\tau_0^2}$ be probability measures of the regression model \eqref{model} and let $\al \in (0,1)$. Assume $||f||_{\infty},||f_0||_{\infty} \leq M_0$ for some $M_0>0$. One can find a constant $C=C(\al,\tau_0^2,M_0)$ such that for any $\varepsilon_n \to 0$, as $n \to 
    \infty$, it holds that
    \[ \{(f,\tau^2)\, : \, D_\al(P_{f,\tau^2},P_{f_0,\tau_0^2}) \leq \varepsilon_n^2\} \subset \{ (f,\tau^2) \, : \, ||f-f_0||_{L^2(P_X)} \leq C\varepsilon_n \, , \, |\tau^2 - \tau_0^2| \leq C\varepsilon_n \}.\]
    In particular, if $\tau_0^2 =1$ is known, we have
    \[ D_\al(f,f_0) \geq \frac{\al}{2} e^{-2M_0^2\al(1-\al)} ||f - f_0||^2_{L^2(P_X)}.\]
\end{lemma}
\begin{proof}
    Let us express $D_\al(P_{f,\tau^2},P_{f_0,\tau_0^2})$, more explicitly. By using the standard formula expressing the $D_\al$-divergence between two univariate Gaussians we get
    \[ D_\al(\mathcal{N}(f(x),\tau^2),\mathcal{N}(f_0(x),\tau_0^2)) = \frac{\al}{2 \tau_\al^2}(f(x)-f_0(x))^2 + \frac{1}{1-\al}\log\frac{\tau_\al}{\tau_0^{1-\al}\tau^\al},\]
    where $\tau_\al^2 := (1-\al)\tau^2 + \al \tau_0^2$, see e.g. \cite{vanerven} eq. (10), one has
    \begin{align*}
        D_\al(P_{f,\tau^2},P_{f_0,\tau_0^2}) &= \frac{1}{\al -1} \log \int e^{(\al-1)D_\al(\mathcal{N}(f(x),\tau^2),\mathcal{N}(f_0(x),\tau_0^2))} \, dP_X(x)\\
        &= \frac{1}{1 - \al}\log\left({\frac{\tau_\al}{\tau_0^{1-\al}\tau^\al}}\right) +\frac{1}{\al -1}\log \int e^{- \frac{\al(1-\al)}{2 \tau_\al^2}(f(x)-f_0(x))^2}\, dP_X(x) \\
        &= \qquad \qquad\ \ (I)\qquad \qquad+\qquad\quad (II).
    \end{align*}
    We have both $(I) \geq 0$ (e.g. by concavity of the logarithm) and $(II) \geq 0$. From the assumption $D_\al(P_{f,\tau^2},P_{f_0,\tau_0^2}) \leq \varepsilon_n^2$ it follows that $(I) \leq \varepsilon_n^2$ and $(II) \leq \varepsilon_n^2$. From Lemma \ref{lem : partI} below, there is a constant $C>0$, such that $|\tau^2 - \tau_0^2| \leq C \varepsilon_n$. Note that $(II) = D_\al(P_{f,\tau_\al^2},P_{f_0,\tau_\al^2}) $. Since $|\tau^2 - \tau_0^2| \leq C \varepsilon_n$, using $\varepsilon_n \to 0$, we get $\tau_\al^2 \leq \tau_0^2 + C(1-\al) \varepsilon_n \leq 2 \tau_0^2$ for $n$ large enough. Therefore $(II) \geq D_\al(P_{f,2\tau_0^2},P_{f_0,2\tau_0^2}) $, where we used that the $\al$-Divergence is a non-increasing function of the variance parameter. Now, using first $1-x \leq - \log x$ and next $1 - e^{-x} \geq x e^{-x}$, leads to
    \begin{align*}
        (II) &\geq \frac{1}{1 - \al} \left[ 1 - \int \exp \left( \frac{\alpha (\alpha -1)}{4\tau_0^2} (f-f_0)^2 \right) \, {dP_X} \right]\\
        &\geq \frac{\al}{4 \tau_0^2} \int e^{\al(\al-1)(f-f_0)^2/(4\tau_0^2)} (f-f_0)^2 \, dP_X \\
        &\geq \frac{\al}{4 \tau_0^2} e^{-2M_0^2\al(1-\al)} ||f - f_0||^2_{L^2(P_X)}.
    \end{align*}
    Therefore, $||f - f_0||_{L^2(P_X)} \leq C \varepsilon_n$ for a suitable constant $C >0$.
\end{proof}
\begin{lemma}\label{lem : partI}
Let $\psi = \psi_\al$ be a function defined on $\R_{>0}$, for $0 < \al < 1$ and $\tau_0 >0$, by
\[ \psi(\tau^2) := \log \left( \frac{(1-\al)\tau^2 + \al \tau_0^2}{(\tau^2)^{1-\al}(\tau_0^2)^\al}\right).\]
There is a constant $c = c(\al,\tau_0^2) >0$, such that for any $\varepsilon_n \to 0$, as $n \to \infty$,
\[\{ \tau^2 \, : \, \psi(\tau^2) \leq c \varepsilon_n^2 \} \subset \{ |\tau^2 - \tau_0^2|\leq \varepsilon_n \}.\]
\end{lemma}

\begin{proof}
    Denoting $\tau_\al^2 := (1-\al)\tau^2 + \al \tau_0^2 $, one gets
    \[ \psi'(\tau^2) = \al(1-\al)(1-\tau_0^2/\tau^2)(\tau_\al^2)^{-1},\]
    so that $\psi'(\tau_0^2)=0$, $\psi$ is decreasing from $+ \infty$ to $\psi(\tau_0^2)=0$ on $(0,\tau_0^2]$ and increasing from $0$ to $+\infty$ on $[\tau_0^2,\infty)$. Also,
    \[ \psi''(\tau^2) = \frac{1 - \al}{(\tau^2)^2} - \frac{(1-\al)^2}{(\tau_\al^2)^2},\]
    so that $\psi''(\tau_0^2) =: 2m >0$. By continuity of $\psi''$, there exists an interval $I:=[\tau_0^2 \pm \eta]$, for some $\eta = \eta(\tau_0^2,\alpha)$, such that $\psi'' \geq m$ on $I$. Set
    \[ c := c(\tau_0^2,\al) = \min( m/2, \psi(\tau_0^2 - \eta), \psi(\tau_0^2 + \eta)).\]
    Let $\tau^2$ verify $\psi(\tau^2) \leq c \varepsilon_n^2$. For $n$ large enough we have $\psi(\tau^2) \leq c$ and by monotonicity of $\psi$ on each side of $\tau_0^2$ one must have $|\tau^2 - \tau_0^2|\leq \eta$, i.e. $\tau^2 \in I$. Recall $\psi(\tau_0^2) = \psi'(\tau_0^2) =0$, computing the Taylor expansion of $\psi$ at $\tau_0^2$ gives, for any $\tau^2 \in I$, and some $\zeta \in I$,
    \[ \psi(\tau^2) =\psi''(\zeta)(\tau^2-\tau_0^2)^2/2 \geq (m/2)(\tau^2-\tau_0^2)^2, \]
    where we used that $\psi'' \geq m$ on $I$. Therefore as soon as $\psi(\tau^2) \leq c \varepsilon_n$ and $n$ is large enough, we have $|\tau^2 - \tau_0^2| \leq \sqrt{2c/m} \varepsilon_n$, which gives the result using $2c/m \leq 1$ by definition of $c$.
    
\end{proof} 

\begin{lemma}\label{lem : conc unknown var}
    Let $f_0 \in \mathcal{F}$ and $\tau_0 > 0$ and $\Pi$ be a probability measure on $\mathcal{F} \times \R_{>0}$. For any $\alpha \in (0,1)$, any positive sequence $(\varepsilon_n)_n$ such that $\varepsilon_n \to 0$, $n \varepsilon_n^2 \to \infty$ and 
    \[ \Pi \left[ \{ (f,\tau^2) \, : \, \lVert f-f_0 \rVert_{\infty} \leq \varepsilon_n \, , \, |\tau^2 - \tau_0^2|\leq \varepsilon_n^2 \} \right] \geq e^{-n\varepsilon_n^2}, \]
    there is a constant $M >0 $ such that, as $n \to \infty$,
    \[ E_{{f_0}}\Pi_{\alpha} \left[\left \{ (f,\tau^2) \, : \,  D_{\alpha}(P_{f,\tau^2},P_{f_0,\tau_0^2}) \geq \frac{M\alpha }{1 - \alpha} \varepsilon_n^2\right\} \, | \, X,Y \right] \to 0 ,\]
    where $E_{f_0}$ denotes the expectation under $P_{f_0,\tau_0^2}$.
\end{lemma}

\section{Additional properties of ReLU DNNs}
We list here some well-known properties of ReLU neural networks, which are very useful for combining different networks:
\begin{itemize}
    \item \emph{Width enlargement}: $\mathcal{F}(L, \mathbf{r}, s) \subseteq \mathcal{F}(L, \mathbf{r'}, s')$ whenever $\mathbf{r} \leq \mathbf{r'}$ component-wise and $s \leq s'$.
    \item  \emph{Composition}: Given $f \in \mathcal{F}(L, \mathbf{r})$ and $g \in \mathcal{F}(L', \mathbf{r'})$ with $r_{L+1} = r_0'$, we can define the network whose realization $g \circ f$ is in the class $\mathcal{F}(L + L', (\mathbf{r}, r_1', \dots, r_{L'+1}'))$.
    \item \emph{Depth synchronization}: To synchronize the number of hidden layers between two networks, we can add layers with the identity matrix as many times as desired,
\[ \mathcal{F}(L, \mathbf{r}, s) \subset \mathcal{F}(L+q, (\underbrace{r_0, \dots, r_0}_{\text{$q$ times}}, \mathbf{r}), s+q \, r_0).\]
    \item \emph{Parallelization}: Given $f \in \mathcal{F}(L, \mathbf{r})$ and $g \in \mathcal{F}(L, \mathbf{r'})$ with $r_0 = r_0'$, one can simultaneously realizes $f$ and $g$ within a joint network whose realization $(f,g)$ is in the class $\mathcal{F}(L, (r_0, r_1 + r_1', \dots , r_{L+1} + r_{L+1}')).$
\end{itemize}

Following is a very useful Lemma giving a bound on the distance between two ReLU neural networks realizations in terms of the distance between their coefficients. These types of inequalities are well established in the literature. The version presented below appears in the proof of Lemma 3 in \cite{suzuki2018adaptivity}.
\begin{lemma}\label{propag1}
Let $f$ and $f^* \in \mathcal{F}(L,\mathbf{r})$ with total number of parameters $T$ and coefficients $(\theta_k)$ and $(\theta_k^*)$. Suppose that for all $k \in [T]$, $|\theta_k| \leq b$, $|\theta_k^*| \leq b$, and $|\theta_k - \theta_k^*| \leq \delta$. Define 
\begin{equation}\label{def : V}
V := \prod_{l=0}^{L} \left( r_l +1 \right),
\end{equation}
then
\[ || f - f^* ||_{L^{\infty}([0,1]^{r_0})} \leq \delta  V (b \vee 1)^L (L+1).\]
\end{lemma}

For reader's convenience  we recall below the approximation results that will be used to prove Theorems \ref{thm : mink} and \ref{thm : besov}, these are respectively borrowed from \cite{nakada2020adaptive} and \cite{suzuki2021deep} where the constants are made explicit.

\begin{lemma}[Theorem 5 in \cite{nakada2020adaptive}]\label{lem : approxmink}
    Let $f \in \mathcal{C}_d^{\beta}([0,1]^d,K)$ and assume that $\dim_M \supp P_X < t$ holds with $t < d$. Let $\varepsilon > 0$. There is a ReLU neural network $\Tilde{f}$ with constant depth $L = L(\beta,d,t)$ and sparsity $ \Tilde{s} \lesssim \varepsilon^{-
    t/\beta}$ such that, when $\varepsilon$ is small enough,
    \[ || f - \Tilde{f}||_{L^{\infty}(P_X)} \leq \varepsilon,\]
    moreover there is a constant $c = c(\beta,d,t)$ such that the weights $\{ \Tilde{\theta}_k , k \in [T] \}$ of $\Tilde{f}$ all satisfy $|\Tilde{\theta}_k| \leq \varepsilon^{-c}$.
\end{lemma}

\begin{lemma}[Proposition 2 in \cite{suzuki2021deep}] \label{lem : approxbesov}
Let $p \in (1,2)$, $\beta \in \R_{>0}^d $ such that $\Tilde{\beta} > 1/p$ and $f \in B_{pp}^{\beta}(1)$. Let $N$ be a large enough positive integer, there exists $\Tilde{f} \in \mathcal{F}(L, (d,r,\dots,r,1),s)$ with $L \lesssim \log N$, $r \lesssim N$ and $s \lesssim NL$ such that 
\[ || f - \Tilde{f} ||_{\infty} \lesssim N^{-\Tilde{\beta}},\]
additionally all the weights in $\Tilde{f}$ are bounded by a universal constant.
    
\end{lemma}

\section{Additional proofs}\label{app : add proofs}

\subsection{A technical Lemma}
\begin{lemma}\label{lem : coeff grand}
    Let $0 <\sigma ,z \leq 1$, $0 < |\Tilde{\theta}| \leq B$ for some $B>1$ and $\theta = \sigma \cdot \zeta$, where $\zeta$ is a random variable on $\R$ with heavy-tailed density $h$ satisfying properties \ref{H1}, \ref{H2} and \ref{H3}. Denoting by $P_{\theta}$ the induced distribution on $\theta$, there is a large enough constant $C>0$ such that,
    \[ P_{\theta}( |\theta - \Tilde{\theta}| \leq z \, , \, |\theta| \leq B) \geq z \exp(-C \log^{1+\kappa}((B + z) \sigma^{-1})).\]
    
\end{lemma}
\begin{proof} 
    Write $B_{-} : = (-B) \vee (\Tilde{\theta} - z)$ and $B^{+} : = B \wedge (\Tilde{\theta} + z)$ so that 
    \[ P_{\theta}( |\theta - \Tilde{\theta}| \leq z \, , \, |\theta| \leq B) = \int_{B_-}^{B_+} \frac{1}{\sigma} h \left( \frac{x}{\sigma} \right)\, dx.\]
    Using the fact that $|\Tilde{\theta}| < B$ one gets, when $\Tilde{\theta} > 0,$
    \[B^+ = B \wedge (\Tilde{\theta} + z) \geq \Tilde{\theta}, \qquad B_- = \Tilde{\theta} -z \]
    and when $\Tilde{\theta} < 0,$
    \[ B^+ = \Tilde{\theta} + z, \qquad B_- = (-B) \vee (\Tilde{\theta} - z) \leq \Tilde{\theta}.\]
    Using the symmetry of $h$, one can reduce the problem to the case $\Tilde{\theta} >0$. Using the positivity of $h$ and the fact that $\sigma \leq 1$, one gets 
    \[ P_{\theta}( |\theta - \Tilde{\theta}| \leq z \, , \, |\theta| \leq B) \geq \int_{\Tilde{\theta} - z}^{\Tilde{\theta}} h \left( \frac{x}{\sigma} \right)\, dx.\]
    Using the monotonicity of $h$, we have
    \[ \int_{\Tilde{\theta} - z}^{\Tilde{\theta}} h \left( \frac{x}{\sigma} \right)\, dx \geq z \times h \left( \frac{ |\Tilde{\theta}| \vee |\Tilde{\theta} - z|}{\sigma} \right) \geq z h((B+z)\sigma^{-1}).\]
    Now simply use the fact that $(B+z)\sigma^{-1} > 1$ and the heavy tail property \ref{H2} of $h$ to conclude.
\end{proof}
\subsection{Proof of Proposition \ref{lemma : approxcomp}}\label{app : approx proofs}
\begin{proof}
First express $f = g_q \circ \dots \circ g_0 : [0,1]^d \to \R$ as the composition of functions defined on hypercubes $[0,1]^{t_i}$. To do this, for all $1 \leq i \leq q-1$, define
\begin{align}\label{def : rescalecomp}
    h_0 := \frac{g_0}{2K} + \frac12, \quad  h_i := \frac{g_i(2K \cdot - K)}{2K} + \frac12, \quad h_q := g_q(2K \cdot - K).
\end{align}
Such that $f_0 = h_q \circ \dots \circ h_0$ and for all $j$, \[
h_{0j} \in \mathcal{C}_{t_0}^{\beta_0}([0,1]^{t_0},1), \quad h_{ij} \in \mathcal{C}_{t_i}^{\beta_i}([0,1]^{t_i},(2K)^{\beta_i}), \quad h_{qj} \in \mathcal{C}_{t_q}^{\beta_q}([0,1]^{t_q},K(2K)^{\beta_q}) .\]

We can now approximate each of the $h_{ij}$ by a network realisation using Lemma \ref{KLt}, before that, let us provide a propagation result that links the approximation quality of $h_{ij}$ to that of the total composition.

\begin{lemma}[Lemma 3 from \citet{JSH}]\label{lem : propacomp}
Let $h_i = (h_{ij})_j$ be functions as in \eqref{def : rescalecomp} with $K \geq 1$. There is a constant $C =C(K, \boldsymbol{\beta}) $ such that for any functions $\Tilde{h}_i := (\Tilde{h}_{ij})_j$ with $\Tilde{h}_{ij} : [0,1]^{t_i} \to [0,1]$, we have,

\[\lVert h_q \circ \dots \circ h_0 - \Tilde{h}_q \circ \dots \circ \Tilde{h}_q \rVert_{L^{\infty}([0,1]^d)} \leq C \sum_{i=0}^q \lVert |h_i - \Tilde{h}_i|_{\infty} \rVert_{L^{\infty}([0,1]^{d_i})}^{\prod_{k=i+1}^q (\beta_k \wedge 1)}.\]
\end{lemma}

We can now apply Lemma \ref{KLt} to each of the $h_{ij}$ separately, using at $i \in \{0 , \dots, q \}$ fixed, $M= M_i := \left \lceil \left( \frac{n}{\log^{\gamma}n} \right)^{\frac{1}{2(2 \beta_i^* + t_i)}}  \right \rceil$. We get for all $(i,j)$, $\Tilde{h}_{ij} \in \mathcal{F}(L_i', (t_i, r_i', \dots r_i',1))$ such that 
\begin{equation}\label{block elem}
    \lVert h_{ij} - \Tilde{h}_{ij} \rVert_{\infty} \lesssim M_i^{-2\beta_i},
\end{equation}
where $L_i' := C_{\ell} \log n$ and $r_i' := C_{w} M_i^{t_i}$ with $C_{\ell}$ and $C_w$ two constants given by lemma \ref{KLt} only depending on $(\boldsymbol{\beta},\mathbf{t})$.

To apply lemma \ref{lem : propacomp} we need to make sure each $\Tilde{h}_{ij}$ with $i < q$ takes values in $[0,1]$. This can be done through the two layer network $(1-(1-x)_+)_+$. Let $L_i = L_i'+2$ and keep calling the new networks $\Tilde{h}_{ij}.$ Since the $h_{ij}$ takes values in $[0,1]$ inequality \eqref{block elem} still holds with the two extra layers.

Now computing the networks $\Tilde{h}_{ij}$ in parallel lends $\Tilde{h}_i = (\Tilde{h}_{ij})_{j \in [d_{i+1}]}$ in the class \[ \mathcal{F}(L_i, (d_i, r_i, \dots ,r_i, d_{i+1})),\] where $r_i := d_{i+1}r_i'$. Now from inequality \eqref{block elem} we immediately get,

\begin{equation}\label{block elem 2}
    \lVert |h_i - \Tilde{h}_i|_{\infty} \rVert_{L^{\infty}([0,1]^{d_i})}^{\prod_{k=i+1}^q (\beta_k \wedge 1)} \lesssim M_i^{-2 \beta_i^*}.
\end{equation}
Finally we compute the composite network $\Tilde{f} :=  \Tilde{h}_q \circ \dots \circ \Tilde{h}_1$ in the class

\begin{equation}\label{rescomp}
\mathcal{F}( L_0 + \dots + L_q, (d,r,\dots,r,1)),
\end{equation}
where $r := \underset{0 \leq i \leq q}{\max} r_i$. Using lemma \ref{lem : propacomp} and \eqref{block elem 2}, for $n$ large enough, we get,
\[ ||f - \Tilde{f}||_{\infty} \lesssim \sum_{i=0}^q M_i^{-2 \beta_i^*} \lesssim \max_i M_i^{-2 \beta_i^*} \leq \phi_n. \]

Now we have for sufficiently large $n$, $L_0+ \dots + L_q = (q+1)C_{\ell}\log n \leq \log^{1+\delta} n$ and $r \leq \lceil \sqrt{n} \rceil$. Thus, for $n$ sufficiently large the space \eqref{rescomp} can be embedded into 
\[ \mathcal{F}(\lceil \log^{1+\delta}n \rceil, (d, \lceil \sqrt{n} \rceil , \dots ,\lceil \sqrt{n} \rceil, 1) ,s),\]
where 
\[ s \lesssim r^2 \log n \lesssim \underset{0 \leq i \leq q}{\max} \left( \frac{n}{\log^{\gamma}n} \right)^{\frac{t_i}{2 \beta_i^* + t_i}} \times \log n \lesssim n \phi_n^2 \log^{1-\gamma}n.  \]
Now for the bound on the coefficients, in view of proposition \ref{prop : bound}, for sufficiently large $n$ every coefficient $\theta$ of the network satisfies 
\[ |\theta| \leq \underset{0 \leq i \leq q}{\max} \lceil M_i^{2(\beta_i+1)} \rceil^2 \leq n^{c_{\beta}},\]
where $c_{\beta} \geq 1$ is a constant only depending on $(\beta_i)_i$. 

This completes the proof.
 \end{proof}

\subsection{Proofs for data with low Minkowski dimension support}\label{app : minkproof}
\begin{proof}[Proof of Theorem \ref{thm : mink}]The proof is very similar to that of Theorem \ref{thm : comp}. Since we will use a sparse approximation result the position of the zeros in the network is not precisely known.

In view of Lemma \ref{lem : conc}  it suffices to show that there exists $C > 0$ such that for sufficiently large $n$,
\[ \Pi\left( \lVert f - f_0 \rVert_{\infty} \leq  \varepsilon_n \right) \geq e^{- C n  \varepsilon_n^2 }.\] 
We apply Lemma \ref{lem : approxmink} for $t > t^*$ and $2\varepsilon= \varepsilon_n := \left( n / \log^{2(1+ \kappa)+1}n\right)^{- \beta / (2 \beta + t)}.$ When $n$ is large enough there is a network $\Tilde{f}_0$ with constant length, number of non-zero parameters $\Tilde{s} \lesssim \varepsilon_n^{-t/\beta}$ and magnitude of coefficients $\max_k |\Tilde{\theta}_k| \leq \varepsilon^{-c} \leq n^c$, such that $|| f_0 - \Tilde{f}_0||_{L^2(P_X)} \leq \varepsilon_n/2$. Since $\Tilde{s} \leq n$ when $n$ is large enough, we can embed $\Tilde{f}_0$ in the space $\mathcal{F}(\log n , (d,n,\dots,n,1) , s)$ where $s \lesssim \Tilde{s} \log n$. Write $(\theta_k)$ (resp. $(\Tilde{\theta}_k)$) for the coefficients of $f$ (resp. $\Tilde{f}_0$). Recall that there is a constant $C_V > 0$, such that 
\begin{equation}\label{boundVmink}
    V := \prod_{l=1}^L (r_l + 1) \leq e^{C_V \log^2 n } .
\end{equation} 
Let $S_0 = \{ k \in [T] \, : \, \Tilde{\theta}_k \neq 0 \}$, recall $|S_0| \leq s \lesssim \varepsilon_n^{-t/\beta} \log n$ and let \[
z_n := \varepsilon_n / (2 n^{cL} V (\log n +1)) .\]
Using triangle inequality and applying Lemma \ref{propag1} we need to bound from below
\[ \Pi\left( \lVert f - f_0 \rVert_{L^2(P_X)} \leq  \varepsilon_n \right) \geq \prod_{k=1}^T \Pi \left( | \theta_k - \Tilde{\theta}_k | \leq z_n \, , \, |\theta_k| \leq n^c \right).\]
 We can follow the proof of Theorem \ref{thm : comp}, splitting the product in the last display on whether or not $k$ is in $S_0$. For $k \in S_0$ we apply Lemma \ref{lem : coeff grand} and get a large enough constant $C_7>0$ such that
\[\prod_{k \in S_0} \Pi \left( | \theta_k - \Tilde{\theta}_k | \leq z_n \, , \, |\theta_k| \leq n^c \right) \geq z_n^s e^{-C_7 s \log^{2(1 + \kappa)}n}. \]
Using the bound \eqref{boundVmink} we have a large enough constant $C_V'$ such that
\[ 2 n^{cL} V (\log n +1 ) \leq e^{C_V' \log^2n}.\]
Finally, use $s \log^{2(1 + \kappa)}n \lesssim \varepsilon_n^{-t/\beta} \log^{2(1 + \kappa)+1}n \leq n \varepsilon_n^2$ to get $C >0$ large enough such that
\[ \prod_{k \in S_0} \Pi \left( | \theta_k - \Tilde{\theta}_k | \leq z_n \, , \, |\theta_k| \leq n^c \right) \geq e^{-C n\varepsilon_n^2 } .\]
For the zero coefficients, $k \notin S_0$ we can simply follow the proof of Theorem \ref{thm : comp} and get
\[ \prod_{k \notin S_0} \Pi \left( | \theta_k | \leq z_n \right) \geq e^{-C n\varepsilon_n^2 },\]
which, according to Lemma \ref{lem : conc} gives us concentration of the tempered posterior in terms of the $\alpha$-Rényi divergence $D_{\alpha}$.
%equivalent to the squared $L^2(P_X)$-distance thanks to the clipping (see Lemma \ref{lem : clip}).

For the concentration of the variational posterior, using the existence of $\Tilde{f}_0 \in \mathcal{F}(L, \br(n),s)$ such that for $n$ large enough $||f_0 - \Tilde{f}_0||_{L^2(P_X)} \leq \varepsilon_n/2$ and $s\log^{2(1+\kappa)}n \lesssim n \varepsilon_n^2$, one can easily follow the steps of the proof of Theorem \ref{thm : varcomp} to get to the result.
\end{proof}
    
\subsection{Proofs for anisotropic Besov spaces}\label{app : proof besov}
\begin{proof}[Proof of Theorem \ref{thm : besov} ]
    We apply Lemma \ref{lem : approxbesov} with $N = \lceil n/ \log^{\gamma}n \rceil^{1 / (2 \Tilde{\beta} + 1 )}$. We get $\Tilde{f}_0 \in \mathcal{F}(L,(d,\Lambda,\dots , \Lambda,1),s)$ such that when $n$ is large enough, $L \lesssim \log n$, $\Lambda \lesssim N$ and 
    \[ s \lesssim N L \lesssim \epsilon_n^{-1/ \Tilde{\beta}} \log n \lesssim n \epsilon_n^2 \log^{1 - \gamma}n .\]
    The network realization $\Tilde{f}_0$ satisfies $\lVert f_0 - \Tilde{f}_0 \rVert_{\infty} \lesssim \epsilon_n / 2$ and the coefficients of $\Tilde{f}_0$ are uniformly bounded by a constant.
    When $n$ is large enough we can embed $\Tilde{f}_0$ into the space $\mathcal{F}( \log^{1 + \delta}n , \br(n) , s)$ and assume that $\max_k |\Tilde{\theta}_k| \leq n$. From here one can follow the proofs of Theorems \ref{thm : varcomp} and \ref{thm : mink} to get the desired results.
\end{proof}

\subsection{Proofs for unknown noise level $\tau_0$ and the case $\al=1$}\label{app : tau inconnu}
\begin{proof}[Proof of Theorem \ref{thm : tau inconnu}]
    In view of Lemma \ref{lem : conc unknown var} (a generalized version of Lemma \ref{lem : conc}) it is enough to show %that \sout{, as $n \to \infty$,}  
    \[ \Pi \left[ \{ (f,\tau^2) \, : \, \lVert f-f_0 \rVert_{\infty} \leq \phi_n \, , \, |\tau^2 - \tau_0^2|\leq \phi_n^2 \} \right] \geq e^{- C n\phi_n^2}, \]
    for some sufficiently large constant $C>0$. Since $\Pi = \Pi_f \otimes \pi_{\tau^2}$, it suffices to verify 
    \begin{align*}
        &\Pi_f(\{ f\, : \, \lVert f-f_0 \rVert_{\infty} \leq \phi_n ) \geq e^{-(C-1)n \phi_n^2}, \\
        &\pi_{\tau^2}(\{ \tau^2 \, : \, |\tau^2 - \tau_0^2| \leq \phi_n^2\}) \geq e^{-n\phi_n^2}.
    \end{align*}
    Since  $\Pi_f$ is the heavy-tailed DNN prior used in Theorem \ref{thm : comp}, the first inequality of the last display as been proved in Section \ref{proof : thmcomp1}. The second inequality immediately follows combining Condition $\eqref{eq : condpriortau}$ and the fact that $\phi_n$ is polynomial in $n^{-1}$. A direct application of Lemma \ref{lem : conc unknown var} leads to contraction of $\Pi_\al[\cdot|X,Y]$ around $(f_0,\tau_0^2)$ at rate $\phi_n$ in $\al$-Rényi divergence. For the clipped posterior $\Pi_\al^B[\cdot|X,Y]$, one uses Lemma \ref{lem : clip} to relate the $\al$-Rényi neighborhood to the $L^2(P_X)$-type one.
\end{proof}

\begin{proof}[Proof of Theorem \ref{thm:true}]
We apply Proposition 1 in \cite{cr24}: the latter shows that, provided that one chooses the prior $\pi_{\ta^2}$ as in \eqref{priorsig}, that the prior on $f$ verifies the 
 prior mass condition, for some $D>0$,
\begin{equation} \label{pmass}
 \Pi_f[\|f-f_0\|_\infty \le \veps_n] \ge e^{-Dn\veps_n^2}, 
\end{equation} 
and that the posterior on $f$ is clipped by $B\ge \|f_0\|_\infty$, then 
\[ E_{f_0,\ta_0}\Pi[\|f-f_0\|_{L^2(P_X)}^2\le M\veps_n^2\,|\,(X,Y)] \to 1.\]
Since \eqref{pmass} is verified within the proof of Theorem \ref{thm : comp} for $\veps_n=c\phi_n$ for some constant $c>0$, and since we work under the conditions of the latter Theorem, the proof is complete.
\end{proof}
 
\subsection{Proofs for general activation functions}\label{app : activ}
We want to prove that if $\rho $ is admissible according to Definition \ref{def : locquad}, all of our results still hold for $\rho$-DNNs. As stated in Section \ref{sec : activ} it is sufficient to show that a result similar to Lemma \ref{propag1} still holds for $\rho$-DNNs and that $\ReLU : x \mapsto \max(0,x)$ can be approximated by a $\rho$-DNN. We provide in Lemma \ref{propag2} a version of Lemma \ref{propag1} for Lispchitz-continuous activations; it follows by a similar proof as for a ReLU activation, tracking the Lipschitz constant throughout the proof, see e.g. the proof of Proposition 1 in \cite{OhnKimActivation}, Appendix B therein.
\begin{lemma}\label{propag2} 
Let $\rho : \R \to \R$ be a $C_\rho$-Lipschitz-continuous activation function and $f,f^*$ two $\rho$--DNNs with width vector $\br = (d, N, \dots, N,1)$, depth $L$ and coefficients $(\theta_k)$ and $(\theta_k^*)$. Suppose that for all $k$, $|\theta_k| \leq b$, $|\theta_k^*| \leq b$, and $|\theta_k - \theta_k^*| \leq \delta$.
\[ || f - f^* ||_{\infty} \leq \delta  L  (C_\rho(b \vee 1)(N+1))^L .\]
\end{lemma}
For the approximation result we use the following Lemma.

\begin{lemma}\label{lem : approxrelu}
    Assume $\rho$ is admissible according to definition \ref{def : locquad}, for any $D \geq 1$ and for any $K \geq 1$ both taken large enough, there exists a $\rho$-DNN, $\tilde{f}_{\ReLU}^{\rho,D}$, with constant width and depth, such that
    \[ \sup_{x \in [-D,D]} |\tilde{f}_{\ReLU}^{\rho,D}(x) - \ReLU(x)| \lesssim \frac{1}{K}.\]
    Additionally all the weights of $\tilde{f}_{ReLU}^{\rho,D}$ satisfy $|\tilde\theta| \lesssim K^2D^6 $.
\end{lemma}
\begin{proof}
From Definition \ref{def : locquad} of an admissible activation function, there exists $a<b \in \R$, such that for any $x \in \R$ and any $K \geq 1$, 
\[ |x \rho(Kx) - \max(ax,bx)| \lesssim 1/K.\]
Notice that for all $x \in \R$, using $a<b$, we have
\begin{equation}\label{eq : reec relu}
\ReLU(x) = \max(0,x) = \frac{1}{b-a}(\max(ax,bx) -ax).   
\end{equation}
From Lemma 1 and Lemma 2 in \cite{kohler2022estimation}, for any $D \geq1$ and any $K \geq 1$ both large enough, there exists two $\rho$-networks $f_{id}^D$ and $f_\times^D$, with constant width and depth, such that the absolute values of their weights is bounded respectively by $B(f_{id}^D) \lesssim KD^2$ and $B(f_{\times}^D) \lesssim K^2D^6$ and such that,
\[\sup_{x\in[-D,D]}|f_{id}^D(x) -x| \lesssim 1/K \qquad \text{and} \qquad \sup_{x,y \in [-D,D]} |f_\times^D(x,y) - xy| \lesssim 1/K.\]
For $K\geq 1 $ large enough and any $x \in [-D,D]$ we have $|f_{id}^D(x)| \leq D + |f_{id}^D(x)-x| \leq 2D$. We set \[ \tilde{f}_{ReLU}^{\rho,D}(x):= \frac{1}{b-a}\left[ f_\times^{2D}(f_{id}^D(x),\rho(Kx)) - af_{id}^D(x)\right], \]
which matches the expression \eqref{eq : reec relu} of the ReLU activation and gives the desired $\rho$-network.
    
\end{proof}
Using Lemma \ref{lem : approxrelu} we can transpose the ReLU-DNN approximations to $\rho$-DNN results.

\begin{lemma}\label{lem : relutorho}
    Assume $\rho$ is Lispchitz-continuous and locally quadratic. Let $f_0 :[0,1]^d \to \R$ be a function such that, there exists a large enough $N \geq 1$ and a ReLU-DNN, $\tilde{f} \in \mathcal{F}(L , (d,N,\dots,N,1))$ with length $L \asymp \log(N)$, satisfying
    \[ || f_0- \tilde{f}||_{\infty} \lesssim N^{-\eta},\]
    where $\eta >0$, and the weights of $\tilde{f}$ are bounded by $N^c$ for some $c\geq1$. Then there exists a $\rho$-DNN, $\tilde{f}^{\rho}$ with width $\tilde{N} \asymp N$ and depth $\tilde{L} \asymp L$, such that
    \[ || f_0 - \tilde{f}^\rho||_{\infty} \lesssim N^{-\eta}, \]
    with weights bounded by $ e^{\tilde{c}\log^2N}$, where $\tilde{c} \geq 1$ is a large enough constant depending only on $c$ and $\eta$.
\end{lemma}    
% \ora{Above one should be a little more specific about the $\lesssim$ signs: one should specify that they do not depend on $f_0$\\ }
\begin{proof}
    By assumption we have a sequence of affine transformations $(A_l)_{1\leq l \leq L+1}$ such that the ReLU-network realization $\tilde{f} \in \mathcal{F}(L,(d,N,\dots,N,1))$ is written
    \begin{equation}\label{eq : reluapprox}
        \tilde{f} : = A_{L+1} \circ \ReLU \circ A_L \circ \dots \circ \ReLU \circ A_1
    \end{equation}
    and is a good approximation of $f_0$. Here ReLU refers to the real valued function $x \mapsto x\vee0$, we recall that in \eqref{eq : reluapprox}, the one-dimensional function ReLU is applied coordinate-wise to the output of the affine transformations $A_l$. Note that for any $x \in [0,1]^d$, we have, $|A_1x|_{\infty} \leq (d+1) N^c \leq N^{c+1}$, for $N$ large enough. Let $K \geq 1$, to be chosen (large enough) below and let $\tilde{f}_1^\rho$ be the $\rho$-DNN of Lemma \ref{lem : approxrelu} with precision $K$ and $D = D_1 := N^{c+1}$, there is a constant $C>0$ such that, for all $x \in [0,1]^d$, as $N$ gets large enough, 
    \[ |\tilde{f}_1^\rho \circ A_1(x) - \ReLU \circ A_1(x)|_{\infty} \leq \frac{C}{K} .\]
    By Lemma \ref{lem : approxrelu}, weights of $\tilde{f}_1^\rho$ are bounded by $K^2 D^6$. Using $|\ReLU(x)| \leq |x|$, we have
    \begin{align*}
        |A_2 \circ\tilde{f}_1^\rho\circ A_1(x)|_{\infty} &\leq (N+1)N^c |\tilde{f}_1^\rho\circ A_1(x)|_{\infty}\\
        &\leq (N+1)N^c ( |\ReLU \circ A_1(x)|_{\infty} + C/K)  \\ 
        &\leq (N+1)N^{c}( D_1 +  C/K) \\
        &\leq  2N^{2(c+1)}(1+1/K) = :D_2. 
    \end{align*}
    By recursion, we can set, for all $1 \leq l \leq L+1$, (with the convention $B_1^{\rho} = A_1$),
    \begin{equation}\label{eq : defDLBL}
        D_l := N^{l(c+1)}2^{l-1}(1+1/K)^{l-1} \qquad \text{and} \qquad B_l^{\rho} := A_l \circ \tilde{f}_{l-1}^\rho \circ A_{l-1} \circ \dots \circ \tilde{f}_1^\rho \circ A_1,
    \end{equation}
    such that the $\rho$-networks $\tilde{f}_l^\rho$ are generated by Lemma \ref{lem : approxrelu}, their weights are bounded by $K^2 D_l^6$, and they satisfy 
    \begin{equation}\label{eq : approxstepBl}
        |\tilde{f}_l^\rho \circ B_{l}^\rho(x) - \ReLU \circ B_{l}^\rho(x)|_{\infty} \leq \frac{C}{K} \quad \text{and} \quad |B_l(x)|_{\infty} \leq D_l.
    \end{equation}
    Now define
    \begin{equation}\label{eq : rhopprox}
        \tilde{f}^\rho : = B_{L+1} = A_{L+1} \circ \tilde{f}_L^\rho \circ A_{L} \circ \dots \circ \tilde{f}_1^\rho \circ A_1.
    \end{equation}
    Equation \eqref{eq : rhopprox} is similar to the definition of $\tilde{f}$, given in equation \eqref{eq : reluapprox}, once one replaces the ReLU activations with their $\rho$-approximations. Recall that each $\rho$-network realization $\tilde{f}_l^\rho : \R \to \R$ has constant width and depth and is applied coordinate wise to the $N$-dimensional outputs $B_{l-1}^{\rho}(x)$. Therefore the $\rho$-network $\tilde{f}_l$ has width $\tilde{N} \asymp N$ and depth $\tilde{L} \asymp L$ and weights bounded by $K^2D_{l}^6$. To conclude it is enough to show that, taking $K$ appropriately large, we get $|| \tilde{f}^\rho - \tilde{f}||_{\infty} \lesssim N^{-\eta}$. Define, for any $1 \leq l \leq L$, (with convention $E_{L} = A_{L+1})$,
    \[ E_l := A_{L+1} \circ \ReLU \circ A_L \circ \dots \circ \ReLU \circ A_{l+1}.\]
    Since ReLU is $1$-Lipschitz (with respect to the sup-norm) and $A_l$ is $N^{c+1}$-Lipschitz, one easily checks that $E_l$ is $\Lambda_l$-Lipschitz with $\Lambda_l \leq N^{(c+1)(L+1-l)}$. Recall \eqref{eq : defDLBL} the definition of $D_l$ and $B_l^\rho$, writing $\tilde{f}^\rho - \tilde{f}$ as a telescopic sum and using \eqref{eq : approxstepBl}, one gets
    \begin{align*}
        |\tilde{f}^\rho(x) - \tilde{f}(x)|_{\infty} &= \left| \sum_{l=1}^LE_l \circ \tilde{f}_l^\rho \circ B_l^\rho(x) - E_l\circ \ReLU \circ B_l^\rho(x)\right|_{\infty} \\
        & \lesssim K^{-1}\sum_{l=1}^L \Lambda_l \leq K^{-1}\sum_{l=1}^LN^{l(c+1)} \leq \frac{N^{(L+1)(c+1)}}{K}.
    \end{align*}
    Now, since $L \asymp\log N $, to get $|| \tilde{f}^\rho - \tilde{f}||_{\infty} \lesssim N^{-\eta}$, it is sufficient to take $K$ large enough, so that 
    $N^{(L +1)(c+1) +\eta} \lesssim K$. We take $K := e^{\tilde{c}_1/2  \log^2N}$ for $\tilde{c}_1 \ge1$ large enough, depending only on $c$ and $\eta$. The weights of $\tilde{f}^\rho$ are all bounded by $D_{L+1}^6K^2 = N^{6(c+1)(L+1)}2^L(1+1/K)^Le^{\tilde{c}_1  \log^2N} \leq e^{\tilde{c} \log^2N}$, where $\tilde{c}\geq 1$ is a large enough constant.
\end{proof}

\begin{theorem}  \label{thm: compactiv}
Consider $\rho$ to be an admissible sigmoid activation according to definition \ref{def : locquad}. Suppose the conditions of Theorem \ref{thm : comp} are satisfied, $\Pi$ being a $\rho$-DNN prior on $f$, with depth $L = \lceil\log^{1+\delta}n\rceil$, width $\br = \br(\sqrt{n})$ and decay $(\sigma_k)$ on the weights given by 
\begin{equation*}
        \log^{3(1+ \delta)} ( i \vee j ) \leq \log(1/ \sigma_l^{(ij)} ) \leq \log^{3(1+ \delta)} n.
\end{equation*}
Then for $\Pi_\al^B[\cdot|X,Y]$ the clipped posterior as defined above, for any $B>\|f_0\|_\infty$, 
\[ E_{f_0} \Pi_\al^B \left[ \left\{ f \, : \,\| f-f_0\|_{L^2(P_X)} \ge M\phi_n' \right\} \, |\, X,Y \right] \to 0,\]
where $\phi_n'$ is the rate given in \eqref{ratephi} with $\gamma' = 3(1+\delta)(1+\kappa)+1.$
\end{theorem}

\begin{proof}
    We check that the steps of the proof of Theorem 2 still hold. One start from the ReLU-approximation given by Lemma \ref{KLt}, this provides the exact setting to apply the transfer Lemma \ref{lem : relutorho} with $\eta = 2 \beta/d$. This shows that any Hölder function in $\mathcal{C}^\beta([0,1^d])$ can be approximated with precision $N^{-2\beta/d}$ with a $\rho$-DNN of depth $L \asymp \log N$, width $r \asymp N$ and weights bounded by $e^{\tilde{c} \log^2N}$, for $N$ large enough (here $N$ plays the role of $M^d$ in Lemma \ref{KLt}). We obtain an equivalent of Prop \ref{lemma : approxcomp} for $\rho$-DNNs following the proof given in Appendix \ref{app : approx proofs}, for $f_0 \in \mathcal{G}(q, \mathbf{d}, \mathbf{t}, \boldsymbol{\beta}, K)$ and $\phi_n'$ as in \eqref{ratephi} with $\gamma' = 3(1+\delta)(1+\kappa)+1$, there is a $\rho$-DNN $\tilde{f}_0$ with depth $L = \lceil\log^{1+\delta}n\rceil$, width $\lceil \sqrt{n} \rceil$ and sparsity $s \lesssim n (\phi_n')^2 \log^{1 - \gamma}n $ such that $|| f_0 - \tilde{f}_0||_\infty \lesssim \phi_n'$, Additionally, all coefficients of $\tilde{f}_0$ satisfy $|\tilde{\theta}_k| \leq e^{c \log^2n} $, where $c \geq 1$ is a large enough constant. Note that the differences from Prop \ref{lemma : approxcomp} is the magnitude of the weight increasing sub-exponentially as $\exp(c \log^2n)$ (instead of polynomially as $n^c$) this induces some extra log-factors in the rate $\phi_n'$ and the decay $\sigma_k$. To conclude we check that the prior mass still holds under these new conditions. Notice that in the proof of Theorem $2$ the magnitude bound is required first in equation \eqref{boundV} which then becomes
    \[2 e^{c L\log^{2}n}(L+1)V \leq e^{C'\log^{3+\delta}n}\]
    and equation \eqref{minoration A} that involves the control of 
    \[ C' \log^{3 + \delta} n - C_0\log^{1 + \kappa}\left((1 + e^{c \log^{2}n} )({\sigma_l^{(ij)}})^{-1} \right).\]
Taking $(\sigma_k)$ such that $\log((\sigma_l^{(ij)})^{-1}) \leq \log^{3(1+\delta)}n$ shows that equation \eqref{minoration A} becomes
\[A \geq \phi_n e^{-C_3 \log^{3(1+\delta)(1+\kappa)}n}.\]
The condition $\log^{3(1+ \delta)} ( i \vee j ) \leq \log(1/ \sigma_l^{(ij)} )$ ensures the rest of the proof follows the one of Theorem \ref{thm : comp}.
\end{proof}

\section{Bound on the coefficients}\label{app : bound}
In this section we give the main steps of the proof of \citet{kohler_full} leading to Lemma \ref{KLt}. Here the goal is to track the bounds on the coefficients of the constructed network in order to obtain a quantitative upper-bound in terms of the number of parameters. For the individual construction of each subnetwork and their approximating capacity we refer to the proof of \citet{kohler_full}.

The goal is to approximate $f \in \mathcal{C}^{\beta}_d ( [0,1]^d,F )$ with precision $M^{- 2\beta}$ on hypercubes of size $M^{-2}$ using local Taylor polynomials of $f$. The network used to find on which of these sub-cubes we are, will lead to the coefficients of highest magnitude through the important quantity
\[ B_M := \lceil M^{2(\beta +1)} \rceil .\]

Note that, in the following, if $g$ is the realization of a neural network with coefficients ${\theta_k}$, we denote the scale of its coefficients as
\[ B(g) := \underset{k}{\max} |\theta_k| .\]

Recall the Lemma
\begin{lemma}[Theorem 2 in \citet{kohler_full}]\label{KLt}
Let $f \in \mathcal{C}^{\beta}_d ( [0,1]^d,F )$ be a function of regularity $\beta >0$, and let $M \geq 2$ be an integer such that the inequality
\[ M^{2 \beta} \geq C(1 \vee F)^{4(\beta +1)}\] holds for a sufficiently large constant $C \geq 1$.

 Let $L,r \in \mathbb{N}$ satisfy
    \[L \geq 5+ \lceil \log_4(M^{2\beta}) \rceil \left( \lceil \log_2(d \vee \lfloor \beta \rfloor +1 ) \rceil +1  \right)\]
    and
    \[r \geq 64 \binom{d+\lfloor \beta \rfloor}{d}2^d d^2 (\lfloor \beta \rfloor +1) M^d.\]
     There exists a ReLU network $\Tilde{f}_{\text{wide}} \in \mathcal{F}(L,(d,r,\dots,r,1))$ such that
    \[|| f - \Tilde{f}_{\text{wide}}||_{\infty} \lesssim (1 \vee F)^{4(\beta +1)} M^{-2 \beta}.\]
\end{lemma}
\begin{proposition}\label{prop : bound}
    Let $B_M := \lceil M^{2(\beta +1)} \rceil$, every coefficient $\theta$ of the network $\Tilde{f}_{\text{wide}}$ of the previous Lemma satisfies
    \[|\theta| \leq \max \{ 2(F \vee 1) e^{2d} , B_M^2 \} .\]
\end{proposition}
Here are the step leading to Proposition \ref{prop : bound}. 
\subsection{Recursive definition of Taylor polynomials}
\begin{lemma}
    Let $f \in \mathcal{C}^{\beta}_d ( [0,1]^d,F )$ and $x_0 \in \R^d$. Define \[T_{f,\beta,x_0}(x) := \sum_{j \in \N^d ; |j| \leq \lfloor \beta \rfloor} (\partial^jf)(x_0)\, \frac{(x-x_0)^j}{j!},\]
    then, for all $ x \in \R^d,$ 
    \[ |f(x) - T_{f,\beta,x_0}(x)| \lesssim ||x-x_0||_{\infty}^{\beta}. \]
\end{lemma}
We conduct the following study on $[-1,1]^d$ to match \cite{kohler_full} results, the same can be done on the unit cube $[0,1]^d$.
Set $\mathcal{P}_2 := \{C_{k,2}\}_{k \in [M^{2d}]}$ a partition of $[-1,1]^d$ into $M^{2d}$ cubes and for all $x \in \R^d$, $(C_{\mathcal{P}_2}(x))_g$ the leftmost corner of the cube in $\mathcal{P}_2$ that contains $x$, from the previous lemma,
\[\forall x \in [-1,1]^d, \quad |f(x) - T_{f,\beta,(C_{\mathcal{P}_2}(x))_g}(x)| \lesssim M^{-2 \beta}. \]

The crux of the proof is to approximate these polynomials by ReLU neural networks, to do so a recursive definition is given. Note that to compute $T_{f,\beta,(C_{\mathcal{P}_2}(x))_g}$ we need to compute about $M^{2d}$ derivatives of $f$. Although it is possible to do so in a single layer of width $M^{2d}$ a significant improvement here from sparse neural classes comes from the fact that between $2$ layers of size $M^d$ in a fully-connected network there are about $M^{2d}$ connections. Therefore the derivatives are set as the weights of the networks, decreasing quadratically the width of the structure. To do this \citet{kohler_full} introduce another partition of broader scale, namely $\mathcal{P}_1 := \{C_{k,1}\}_{k \in [M^{d}]}$ a partition of $[-1,1]^d$ into $M^{d}$ cubes.

Now for any $i \in [M^d],$ let $\{ \Tilde{C}_{j,i} \}_{j \in [M^d]}$ be the cubes of $\mathcal{P}_2$ that are in $C_{i,1}$. Thus $\mathcal{P}_2 = \{ \Tilde{C}_{i,j} \}_{1 \leq i,j \leq M^d}$ and for all $x \in \R^d$,
\begin{equation}\label{Tloc}
    T_{f,\beta,(C_{\mathcal{P}_2}(x))_g}(x) = \sum_{1 \leq i,j \leq M^d} T_{f,\beta,(\Tilde{C}_{i,j})_g}(x) \, \mathbf{1}_{\Tilde{C}_{j,i}}(x)
    \end{equation}

In order to compute this as a neural network, the expression \eqref{Tloc} of the local Taylor polynomial of $f$ can be written recursively as follows:
\begin{enumerate}
    \item first find in which broader cube of $\mathcal{P}_1$ the point $x$ is located and compute derivatives of $f$ on each smaller cube of size $1/M^2$ inside,
    \begin{align*}
        \phi_{1,1} &:= x ,\\
        \phi_{2,1} &:= \sum_{i =1}^{M^d} (C_{i,1})_g \, \mathbf{1}_{C_{i,1}}(x) , \\
        \phi_{3,1}^{(l,j)} &:= \sum_{i =1}^{M^d} (\partial\,^l f)( (\Tilde{C}_{j,i})_g) \, \mathbf{1}_{C_{i,1}}(x),
    \end{align*}
    for all $j \in [M^d]$ and $|l| \leq \lfloor \beta \rfloor.$
    \item then find the small cube of $\mathcal{P}_2$ containing $x$ and gather the corresponding derivatives from the first step; to do this first note that for all $1 \leq k,i \leq M^d$,
    \begin{equation}\label{v}
        v_k := (\Tilde{C}_{k,i})_g - (C_{i,1})_g
        \end{equation}
        takes values in $\{0, 2/M^2 , \dots , 2(M-1)/M^2 \}$ and there is $j \in [M^d]$ such that,
    \[C_{\mathcal{P}_2}(x) = \mathcal{A}^{(j)} := \{ x \in \R^d , \, \phi_{2,1}^{(k)} + v_j^{(k)} \leq x^{(k)} < \phi_{2,1}^{(k)} + v_j^{(k)} + 2/M^2, \quad \forall k \in\{1,\dots,d\} \}.\]
    Then set,
    \begin{align*}
        \phi_{1,2} &:= x, \\
        \phi_{2,2} &:= \sum_{j=1}^{M^d}(\phi_{2,1} + v_j) \, \mathbf{1}_{\mathcal{A}^{(j)}}(\phi_{1,1}), \\
        \phi_{3,2}^{(l)} &:= \sum_{j=1}^{M^d} \phi_{3,2}^{(l,j)} \, \mathbf{1}_{\mathcal{A}^{(j)}}(\phi_{1,1}).\\
    \end{align*}

    \item finally compute the local Taylor polynomial as
    \[ \phi_{1,3} := \sum_{j \in \N^d ; |j| \leq \lfloor \beta \rfloor} \frac{\phi_{3,2}^{(j)}}{j!} \, (\phi_{1,2} - \phi_{2,2})^j.\]
\end{enumerate}

\begin{lemma}
       $\phi_{1,3} = T_{f,\beta,(C_{\mathcal{P}_2}(x))_g}(x)$
    \end{lemma}
    
In the next step we construct an approximation of $\phi_{1,3}$ by ReLU neural networks. This approximation will be correct as soon as $x$ is not too close to one of the boundaries of a $\mathcal{P}_2$ cube. More precisely, if $C$ is a cube and $\delta$ is a strictly positive real number, let $C_{\delta}^0$ be the set of points in $C$ at a distance of at least $\delta$ from the boundary of $C$.

\begin{lemma}\label{approxgrille}
    There is $\Tilde{f} \in \mathcal{F}(L,(d,N,\dots,N,1))$ such that $\lVert \Tilde{f} \rVert_{\infty} \leq 2(F \vee 1)e^{2d}$ and \[\forall x \in \bigcup_{j=1}^{M^{2d}} (C_{j,2})_{1/M^{2\beta+2}}^0, \quad |f(x)-\Tilde{f}(x)| \lesssim M^{-2\beta}.\]
    Furthermore, $L \asymp \log M$, $N \lesssim M^d$ and $B(\Tilde{f}) \leq F \vee B_M^2$ with $B_M := \lceil M^{2\beta+2} \rceil$.
\end{lemma}

The boundary condition arises from the fact that through ReLU networks, only continuous piecewise affine functions can be realized. Therefore, it is necessary to approximate the indicator functions in the definition of $\phi_{1,3}$. Moreover, we will see that the accuracy of the approximation of these indicators controls the magnitude of the largest coefficients of the network.

\subsection{Auxiliary networks for the construction of \texorpdfstring{$\phi_{1,3}$}{phi}}

To approximate polynomial functions through ReLU networks the first step is to approximate the multiplication between two numbers, this construction comes from \citet{JSH}
\begin{lemma}\label{mult}
For any positive integer $R$ there is $\hat{f}_{mult} \in \mathcal{F}(R,18)$ with $B(\hat{f}_{mult}) \leq 4$, such that,
\[\forall x,y \in [-1,1], \quad |\hat{f}_{mult}(x,y) - xy| \leq 4^{-R}.\]
    
\end{lemma}
Starting from Lemma \ref{mult}, it is quite easy to provide an approximation result for polynomials in several variables. Let $N \in \mathbb{N}$, denote the set of $d$-variables polynomials of total degree at most $N$ by,
\[\mathcal{P}_N := \Vect \left( \, \prod_{k=1}^d (x^{(k)})^{r_k}\, , \, r_1 + \dots + r_d \leq N \right).\]
Since $\dim \mathcal{P}_N = { \binom{d+N}{d} }$, write $m_1, \dots , m_{\binom{d+N}{d}}$ all the monomials in $\mathcal{P}_N$. Let $r_1, \dots , r_{\binom{d+N}{d}}$ be real numbers, for all $x \in [-1,1]^d$ and $y_1 , \dots , y_{\binom{d+N}{d}} \in [-1,1],$ set
\[\mathfrak{p}(x,y_1,\dots,y_{\binom{d+N}{d}}) := \sum_{k=1}^{\binom{d+N}{d}} r_k \, y_k  \, m_k(x).\]
\begin{lemma}\label{polynome}
    Let $R \geq \log_4(2 \times 4^{2(N+1)})$ be an integer and $r_1, \dots , r_{\binom{d+N}{d}}$ real numbers, set $\overline{r(\mathfrak{p})} := \max |r_k|$. There is a fully connected neural network of length $R \lceil \log_2(N+1) \rceil $ and width $18 (N+1) { \binom{d+N}{d}}$ whose realisation $\hat{f}_{\mathfrak{p}}$ is such that $B(\hat{f}_{\mathfrak{p}}) \leq 4 \vee \overline{r(\mathfrak{p})}$, and,
    \[\left|\hat{f}_{\mathfrak{p}}(x,y_1,\dots,y_{{\binom{d+N}{d}}}) - \mathfrak{p}(x,y_1,\dots,y_{{\binom{d+N}{d}}}) \right| \lesssim \overline{r(\mathfrak{p})} 4^{-R},\]
    for all $x \in [-1,1]^d, y_1 , \dots , y_{\binom{d+N}{d}} \in [-1,1].$
\end{lemma}

 As seen in the previous Section one also needs to approximate indicator functions, this is where the coefficients of highest magnitude will appear.

\begin{lemma}\label{indicatrice}
    Let $R \in \N$ and $a,b \in \R^d$ such that for all $b^{(i)} - a^{(i)} \geq 2/R$ for all $i \in \{1, \dots , d \}.$ Let
    \[K_{1/R} := \left\{ x \in \R^d , \, x^{(i)} \notin [a^{(i)},a^{(i)} + 1/R[ \, \cup \, ]b^{(i)}- 1/R, b^{(i)}], \text{ for all } i \in [d]\right\}.\]
    \begin{enumerate}
        \item There is $\hat{f}_{ind [a,b[} \in \mathcal{F}(2,2d)$ with
        $B(\hat{f}_{ind [a,b[}) \leq \max(R, 1/R, |a|_{\infty}, |b|_{\infty})$ such that
        \[\forall x \in K_{1/R}, \quad \hat{f}_{ind [a,b[}(x) = \mathbf{1}_{[a,b[}(x), \]
        and
        \[ \forall x \in \R^d, \quad |\hat{f}_{ind [a,b[}(x) - \mathbf{1}_{[a,b[}(x)| \leq 1.\]

        \item Let $|s| \leq R$, there is $\hat{f}_{test}(.\,,a,b,s) \in \mathcal{F}(2,2(2d+2))$ with
        $B(\hat{f}_{test}) \leq \max(R^2, 1/R, |a|_{\infty}, |b|_{\infty})$ such that
        \[\forall x \in K_{1/R}, \quad \hat{f}_{test}(x,a,b,s) = s \mathbf{1}_{[a,b[}(x), \]
        and
        \[ \forall x \in \R^d, \quad |\hat{f}_{test}(x,a,b,s) - s \mathbf{1}_{[a,b[}(x)| \leq |s|.\]
    \end{enumerate}
\end{lemma}
Also it is clear that one can realise the identity function of $\R^d$ with a ReLU network of arbitrary length, write $f_{id}^{\ell}$ such network of length $\ell \geq 2$.
Choose $l_1, \dots , l_{ \binom{d + \lfloor \beta \rfloor}{d} }$ such that 
\[ \left\{ l_1, \dots , l_{ \binom{d + \lfloor \beta \rfloor}{d} } \right\} = \{ (s_1,\dots,s_d) \in \N^d \, , \, s_1 + \dots + s_d \leq \lfloor \beta \rfloor \}. \]
Here are the networks deployed to approximate $\phi_{1,3}$, where for the $\hat{f}_{ind}$ and $\hat{f}_{test}$ networks we use Lemma \ref{indicatrice} with $R = B_M := \lceil M^{2\beta + 2} \rceil$. Let $\mathbf{J} := (1,\dots,1)^T \in \R^d$,

\begin{align*}
    \hat{\phi}_{1,1} &:= f_{id}^2(x), \\
    \hat{\phi}_{2,1} &:= \sum_{i=1}^{M^d} (C_{i,1})_g \, \hat{f}_{ind \, C_{i,1}}(x),\\
    \hat{\phi}_{3,1}^{(l,j)} & := \sum_{i=1}^{M^d}(\partial\,^l f)((\Tilde{C}_{j,i})_g) \hat{f}_{ind \, C_{i,1}}(x), \\
    \hat{\phi}_{1,2} &:= f_{id}^2(\hat{\phi}_{1,1}), \\
    \hat{\phi}_{2,2}^{(k)} &:= \sum_{i=1}^{M^d} \hat{f}_{test}(\hat{\phi}_{1,1},\hat{\phi}_{2,1} + v_i, \hat{\phi}_{2,1} + v_i + 2/M^2 \cdot \mathbf{J}, \hat{\phi}_{2,1}^{(k)}+ v_i^{(k)} ), \\
    \hat{\phi}_{3,2}^{(l)} &:= \sum_{i=1}^{M^d} \hat{f}_{test}(\hat{\phi}_{1,1},\hat{\phi}_{2,1} + v_i, \hat{\phi}_{2,1} + v_i + 2/M^2 \cdot \mathbf{J}, \hat{\phi}_{3,1}^{(l,i)} ).
\end{align*}

 Now, for any $u \in \{1 , \dots , { \binom{d + \lfloor \beta \rfloor}{d} }\}$ set $y_u := \hat{\phi}_{3,2}^{(l_u)} $ and $r_u := 1/(l_u)!$, the approximating network of $\phi_{3,1}$ is given by Lemma \ref{polynome} with 
 \[R = B_{M,\mathfrak{p}} := \lceil \log_4(2 \cdot 4^{2(\beta +1)} \cdot (2 \vee F)^{2(\beta +1)}) \rceil,\]
 through
 \[ \hat{\phi}_{1,3} = \hat{f}_{\mathfrak{p}}(\hat{\phi}_{1,2} - \hat{\phi}_{2,2}, y_1,\dots y_{\binom{d + \lfloor \beta \rfloor}{d}}). \]

 The network $\hat{\phi}_{1,3}$ satisfies conditions of Lemma \ref{approxgrille}, we only provide the bound on the coefficients and refer to \citet{kohler_full} for a proof on the approximating properties. It is clear that $B(\hat{\phi}_{1,1}) = B(\hat{\phi}_{1,2}) = 1 $. Using Lemma \ref{indicatrice} one gets
\[B(\hat{\phi}_{2,1}) \leq \max \left\{ \max_i |(C_{i,1})_g|_{\infty} , B_M^2, 1/B_M \right\} \leq \max(2/M , B_M^2) \leq B_M^2,\]
and for all $l \in \{1 , \dots , { \binom{d + \lfloor \beta \rfloor}{d} }\}$ and $j \in [M^d],$
\[B(\hat{\phi}_{3,1}^{(l,j)}) \leq \max \left\{ \max_i | (\partial\,^l f)((\Tilde{C}_{j,i})_g)|, B_M^2, 1/B_M \right\} \leq F \vee B_M^2.\]

From the proof of \citet{kohler_full} it appears that the networks previously defined $(\hat{\phi}^{(k)}_{2,1}, \hat{\phi}_{3,1}^{(l,j)})$ are all in the set of bounded functions $ \{ g \, : \, ||g||_{\infty} \leq 1 \vee F \}$, thus using Lemma \ref{indicatrice} one gets, for all $k \in \{1,\dots,d\} $ and $l \in \{1 , \dots , { \binom{d + \lfloor \beta \rfloor}{d} }\}$,
\[
  B(  \hat{\phi}_{2,2}^{(k)} ) \leq F \vee B_M^2 \quad \text{and} \quad B(\hat{\phi}_{3,2}^{(l)}) \leq F \vee B_M^2. \]

One can check that $|\hat{\phi}_{1,2} - \hat{\phi}_{2,2}| \leq 2$ and $\lVert \hat{\phi}_{3,2} \rVert_{\infty} \leq F$. Now since $\overline{r(\mathfrak{p})} = \max |1/(l_u)!| \leq 1$ and $B_M \geq 4$ from $M \geq 2$, Lemma \ref{polynome} yields,
\[ B(\hat{\phi}_{1,3}) \leq \max \{ F, B_M^2, B_{M,\mathfrak{p}} \} \leq F \vee B_M^2.\]

\subsection{Approximation of \texorpdfstring{$w_{\mathcal{P}_2}(x) \cdot f(x)$}{a}}
The next step in the proof deals with handling the case where $x$ is such that the approximation provided by Lemma \ref{approxgrille} is no longer valid. To address this, define a weight function $w_{\mathcal{P}_2}$  for all $x \in \R^d$ by

\begin{equation}\label{tente1}
w_{\mathcal{P}_2}(x) := \prod_{k=1}^d \left( 1 - M^2 \left| (C_{\mathcal{P}_2}(x))_g^{(k)} + \frac{1}{M^2} - x^{(k)} \right|\right)_+.
\end{equation}

This function reaches its maximum at the center of the cube of $\mathcal{P}_2$ containing $x$ and is zero outside. Moreover, it takes very small values near the boundaries of the cube. More precisely, we have,
\begin{equation}\label{tente2}
\text{for all } x \in \bigcup_{j=1}^{M^{2d}} C_{j,2} \setminus (C_{j,2})_{1/M^{2 \beta +2 }}^0, \quad w_{\mathcal{P}_2}(x) \leq {M^{- 2 \beta}}.
\end{equation}
In this step we want to construct a network that will realise a good approximation of $ w_{\mathcal{P}_2}(x) \cdot f(x)$ on the whole cube. In the final step we will aggregate these approximations using a partition of unity argument to get back to $f$.

\begin{lemma}\label{approxtente}
    There is $\Tilde{f} \in \mathcal{F}(L,(d,N,\dots,N,1))$ with $L \lesssim \log M$, $N \lesssim M^d$ and $B(\Tilde{f}) \leq B_M^2 \vee F$, such that
    \[\forall x \in [-1,1]^d, \quad |\Tilde{f}(x) - w_{\mathcal{P}_2}(x) \cdot f(x) | \lesssim M^{-2\beta} .\]
\end{lemma}
First approximate $w_{\mathcal{P}_2}$ with wide neural networks.
\begin{lemma}\label{approxpoids}
    There is a network $\hat{f}_{w_{\mathcal{P}_2}} \in \mathcal{F}(L,(d,N,\dots,N,1))$ such that $\lVert \hat{f}_{w_{\mathcal{P}_2}} \rVert_{\infty} \leq 2$ and
    \[\forall x \in \bigcup_{i=1}^{M^{2d}} (C_{i,2})_{1/M^{2\beta +2}}^0, \quad |\hat{f}_{w_{\mathcal{P}_2}}(x) - w_{\mathcal{P}_2}(x) | \leq M^{-2 \beta},\]
    with $L \lesssim \log M$, $N \lesssim M^d$ and $B(\hat{f}_{w_{\mathcal{P}_2}}) \leq B_M^2 \vee F.$
\end{lemma}
From \eqref{tente1} it is clear that such network exists using $\hat{\phi}_{2,2}$ to get $(C_{\mathcal{P}_2}(x))_g$ and $\hat{\phi}_{1,2}$ to get $x$. Those yield good approximations since $x$ is ``far'' from the boundary of the cube and from previous calculations, \[B(\hat{f}_{w_{\mathcal{P}_2}}) \leq \max \{ 4, B(\hat{\phi}_{1,2}) , B(\hat{\phi}_{2,2}), M^2, 1/M^2\} \leq B_M^2 \vee F.\]
Lemmas \ref{approxgrille} and \ref{approxpoids} already show that there is a network that will approximate correctly $w_{\mathcal{P}_2}(x) \cdot f(x)$ whenever $x$ is at least $M^{-2(\beta +1)}$ away from the boundaries of $C_{\mathcal{P}_2}(x)$. The idea now is to change a bit this network so that it is zero whenever $x$ is too close to the boundaries. The resulting network will be a good approximation of $w_{\mathcal{P}_2}(x) \cdot f(x)$ on the whole cube because of \eqref{tente2}. The following lemma defines a network that checks the relative position of $x$ with respect to the boundaries of $C_{\mathcal{P}_2}(x)$.

\begin{lemma}\label{check}
    There is $\hat{f}_{check \mathcal{P}_2} \in \mathcal{F}(5,(d,N,\dots,N,1))$ such that $\lVert \hat{f}_{check \mathcal{P}_2} \rVert_{\infty} \leq 1$ and
    \[\forall x \notin \bigcup_{i=1}^{M^{2d}} (C_{i,2})_{1/M^{2\beta +2}}^0 \setminus (C_{i,2})_{2/M^{2\beta +2}}^0, \quad \hat{f}_{check \mathcal{P}_2}(x) = \mathbf{1} \big( \bigcup_{j=1}^{M^{2d}} C_{j,2} \setminus (C_{j,2})_{1/M^{2 \beta +2 }}^0 \big),\]
    with $N \asymp M^d$ and $B(\hat{f}_{check \mathcal{P}_2}) \leq B_M^2.$
\end{lemma}
First find which cube of $\mathcal{P}_1$ contains $x$ using Lemma \ref{indicatrice} with $R = B_M$,
\[ \hat{f}_1 := 1 - \sum_{i=1}^{M^d} \hat{f}_{ind(C_{i,1})_{1/M^{2 \beta + 2}}^0},\]
previous calculations have shown that $B(\hat{f}_1) \leq B_M^2$. Then find which smaller cube of $\mathcal{P}_2$ contains $x$ using lemma \ref{indicatrice} again with $R = B_M$. Recall $\mathbf{J} = (1,\dots,1)^T \in \R^d$,
\[ \hat{f}_2(x) := 1 - \sum_{j=1}^{M^d} \hat{f}_{test}(f_{id}^2(x), \hat{\phi}_{2,1} + v_j + M^{-2(\beta +1)} \cdot \mathbf{J}, \hat{\phi}_{2,1} + v_j + 2/M^2 \cdot \mathbf{J} +M^{-2(\beta +1)} \cdot \mathbf{J},1).\]
From previous calculations we know $B(\hat{f}_2) \leq \max \{ B_M^2, B(\hat{\phi}_{2,1}) \} \leq B_M^2.$ Recall $\rho(x) = \max\{0,x\}$ and define \[ \hat{f}_{check\mathcal{P}_2}(x) := 1 - \rho(1 - \hat{f}_2(x) - f_{id}^2(\hat{f}_1(x))).\]
This network satisfies conditions of Lemma \ref{check}. 

To construct a network satisfying the conditions of Lemma \ref{approxtente} let $\hat{f}_{\mathcal{P}_2}$ be the network defined by Lemma \ref{approxgrille} and $B_{true} := 2(1 \vee F) e^{2d}$ be an upper bound on its supremum norm. Set
\[\hat{f}_{\mathcal{P}_2,true} := \rho(\hat{f}_{\mathcal{P}_2} - B_{true} \cdot \hat{f}_{check\mathcal{P}_2}) - \rho(-\hat{f}_{\mathcal{P}_2} - B_{true} \cdot \hat{f}_{check\mathcal{P}_2}),\]
this network is zero whenever $x$ is less than $M^{-2(\beta +1)}$ close from the boundaries of $C_{\mathcal{P}_2}$ and otherwise it will be a good approximation of $f$. Thus $\Tilde{f} = \hat{f}_{mult}(\hat{f}_{\mathcal{P}_2,true}, \hat{f}_{w_{\mathcal{P}_2}})$ is always a good approximation of $w_{\mathcal{P}_2}(x) \cdot f(x)$ and satisfies conditions of Lemma \ref{approxtente}. Indeed, we have
\[B(\hat{f}_{\mathcal{P}_2,true}) \leq \max \{ B_{true}, B(\hat{f}_{\mathcal{P}_2}), B(\hat{f}_{check\mathcal{P}_2}) \} \leq \max \{ 2(F \vee 1) e^{2d} , B_M^2 \}, \]
and \[B(\Tilde{f}) \leq \max \{ B(\hat{f}_{\mathcal{P}_2,true}) , B(\hat{f}_{w_{\mathcal{P}_2}}) \} \leq \max \{ 2(F \vee 1) e^{2d} , B_M^2 \}.\]

\subsection{Partition of unity through \texorpdfstring{$w_{\mathcal{P}_2}$}{a}}
For the final step, construct a partition of unity using the $w_{\mathcal{P}_2}$ weight functions. To do so shift the partitions $\mathcal{P}_1$ and $\mathcal{P}_2$ by shifting at least one of the coordinates by $1/M^2$. This leads to two sets of $2^d$ partitions $\{ P_{1,v} , v \in \{1,\dots, 2^d \}\}$ and $\{ P_{2,v} , v \in \{1,\dots,2^d \}\}$. Enlarging the original hypercube if necessary the resulting family of weight functions $\{ w_{\mathcal{P}_2,v} \}$ is a partition of unity on $[-1,1]^d$. If we set $\hat{f}_v$ to be the network given by Lemma \ref{approxtente} for the partition $\mathcal{P}_{2,v}$ then 
$f_{wide} := \sum_{v=1}^{2^d} 
\hat{f}_v $ yields a good approximation of $f = \sum_v w_{\mathcal{P}_{2,v}} f$ and satisfies all the conditions of Lemma \ref{KLt}, in particular
\[B(f_{wide}) \leq \max \{ 2(F \vee 1) e^{2d} , B_M^2 \}.\]

\bibliography{bib_ht_dnn}

\end{document}